\documentclass{article}

\usepackage{enumitem}

\usepackage[final]{neurips_2025}

\usepackage[utf8]{inputenc} % allow utf-8 input
\usepackage[T1]{fontenc}    % use 8-bit T1 fonts
\usepackage[colorlinks,linkcolor=black]{hyperref}       % hyperlinks
\usepackage{url}            % simple URL typesetting
\usepackage{booktabs}       % professional-quality tables
\usepackage{amsfonts}       % blackboard math symbols
\usepackage{nicefrac}       % compact symbols for 1/2, etc.
\usepackage{microtype}      % microtypography
\usepackage{xcolor}         % colors

\usepackage{bbm}
\usepackage{adjustbox}
\usepackage{amsmath,epsfig,epsf,psfrag,natbib}
\usepackage{subfigure}
\usepackage{booktabs}
\usepackage{lscape}
\usepackage{color,latexsym,amssymb,amsthm,amscd,graphicx}
\usepackage{mathrsfs}
\usepackage{multirow}
\usepackage{diagbox}
\usepackage[ruled,vlined,lined,boxed,linesnumbered]{algorithm2e}

\newtheorem{condition}{{ Condition}}
\newtheorem{lemma}{\sc Lemma}
\newtheorem{theorem}{{\sc Theorem}}

\newtheorem{assumption}{{\sc Assumption}}

\newtheorem{proposition}{{\sc Proposition}}
\newtheorem{definition}{{\sc Definition}}

\newcommand{\ba}{\begin{eqnarray}}
	\newcommand{\ea}{\end{eqnarray}}
\newcommand{\bas}{\begin{eqnarray*}}
	\newcommand{\eas}{\end{eqnarray*}}
\newcommand{\ben}{\begin{enumerate}}
	\newcommand{\een}{\end{enumerate}}
\newcommand{\bit}{\begin{itemize}}
	\newcommand{\eit}{\end{itemize}}
	
\newcommand{\e}{ { \mathbb{E}}}

\newcommand{\bs}{ {\bf s}}

\usepackage{authblk}

\title{Conformal Prediction Beyond the Horizon: Distribution-Free Inference for Policy Evaluation}

\author{
 Feichen Gan, Youcun Lu, Yingying Zhang\textsuperscript{*}
     and Yukun Liu\\
  KLATASDS - MOE, School of Statistics, East China Normal University\\
}

\begin{document}

\maketitle
\footnotetext[1]{\textsuperscript{*}Corresponding author: yyzhang@fem.ecnu.edu.cn.}

% \footnotetext[1]{\textsuperscript{*}Yingying Zhang and Yukun Liu are co-corresponding authors. Emails: yyzhang@fem.ecnu.edu.cn, ykliu@sfs.ecnu.edu.cn.}
% Emails: yyzhang@fem.ecnu.edu.cn, ykliu@sfs.ecnu.edu.cn.

\begin{abstract}
Reliable uncertainty quantification is crucial for reinforcement learning (RL) in high-stakes settings. We propose a unified conformal prediction framework for infinite-horizon policy evaluation that constructs distribution-free prediction intervals {for returns} in both on-policy and off-policy settings. Our method integrates distributional RL with conformal calibration, addressing challenges such as unobserved returns, temporal dependencies, and distributional shifts. We propose a modular pseudo-return construction based on truncated rollouts and a time-aware calibration strategy using experience replay and weighted subsampling. These innovations mitigate model bias and restore approximate exchangeability, enabling uncertainty quantification even under policy shifts. Our theoretical analysis provides coverage guarantees that account for model misspecification and importance weight estimation. Empirical results, including experiments in synthetic and benchmark environments like Mountain Car, show that our method significantly improves coverage and reliability over standard distributional RL baselines.	
\end{abstract}

\section{Introduction}

\paragraph{Motivation.} As reinforcement learning (RL) are increasingly  deployed in high-stakes domains, such as healthcare, robotics, and autonomous systems, robust uncertainty quantification becomes essential. While traditional policy evaluation methods focus on estimating the expected return, this is insufficient when decisions must account for risk, reliability, and rare outcomes. For example, in clinical decision-making, a treatment policy may appear beneficial on average but hide adverse effects for specific patient subgroups. Even in less safety-sensitive applications such as recommendation systems or finance, overlooking uncertainty can lead to unstable behavior and degraded user experience. Prediction intervals (PIs) for returns offer a principled way to quantify uncertainty, enabling risk-aware planning and safer deployments.

This paper focuses on constructing valid PIs for \textbf{infinite-horizon} RL settings, where the return is defined as the sum of discounted rewards. In on-policy settings, PIs help assess the variability of returns under the current policy, enabling more robust policy improvement and risk-sensitive exploration. In off-policy scenarios, where evaluating a new policy offline based on
an observational dataset, PIs serve to gauge the reliability of point estimation from historical data. By constructing PIs for the return, our approach improves the transparency, reliability, and robustness of RL systems across a wide range of domains.

\paragraph{Challenges.} Constructing valid PIs for returns in RL is closely tied to estimating the full return distribution, as studied in Distributional RL (DRL). In principle, conditional quantiles from this distribution can be used to form PIs.  However, existing DRL-based approaches often suffer from model misspecification, leading to biased or inconsistent return distribution estimates and a lack of formal statistical guarantees.  %To address this, we build on the framework of conformal prediction, which provides a flexible, model-agnostic method for constructing PIs with finite-sample coverage guarantees. 
To address this, {building on the framework of conformal prediction, we propose  a flexible, model-agnostic methodology}  for constructing PIs with  asymptotic  coverage guarantees. Applying conformal prediction to the  infinite-horizon RL setting requires substantial methodological innovation, as it poses several fundamental challenges: 
\begin{itemize}[leftmargin=*]
    \item \textbf{Unobserved Returns.}  In infinite-horizon RL, the return cannot be directly observed, since in practice only finite-horizon trajectories (of length $T$) are available and future rewards beyond $T$ are unobserved. Although  mitigated by discounting,  the truncation error  remains non-negligible in offline settings when $T$ is moderate, making it challenging to evaluate prediction errors or calibrate uncertainty. 
    \item  \textbf{Temporal Dependence.} RL data are inherently sequential, violating the exchangeability assumption required by standard conformal prediction methods. 
    \item \textbf{Distribution Shifts.} In on-policy setting, discrepancies over time lead to  complex covariate shift in the state distribution. In off-policy evaluation, discrepancies between the behavior policy and the target policy also lead to covariate shift in the state-action distribution. 
    %\item \textbf{Model Misspecification.} Return distributions estimated by distributional RL methods can be biased or inconsistent, undermining the reliability of naive quantile-based intervals.{\red{[Delete this item]}}
\end{itemize}

\paragraph{Contributions.} We propose a novel, distribution-free method that integrates conformal prediction with distributional RL to construct prediction intervals for infinite-horizon returns under both on-policy and off-policy settings. Our contributions are as follows: \emph{(1) Pseudo-Return Construction.} We develop a modular approximation scheme for unobserved returns, combining truncated rollouts with tail sampling from learned return distributions. This design is inspired by temporal-difference learning and enables calibration despite partial observability. \emph{(2) Calibration via Experience Replay.} To mitigate temporal dependence and approximate exchangeability, we adopt experience replay and apply random subsampling to the calibration set. This design recovers approximate exchangeability, enabling valid conformal calibration. \emph{(3) Time-Aware Weighted Subsampling.} We address distribution shifts both over time and between policies, using a simple, weighted subsampling scheme. This enables valid calibration in off-policy settings and improves efficiency in on-policy scenarios. \emph{(4) Theoretical Guarantees.} We establish asymptotic lower bounds on coverage using Wasserstein metrics, characterizing how model bias and density ratio estimation affect conformal validity. \emph{(5) Empirical Validation.} We demonstrate the effectiveness of our method through empirical studies on synthetic and the Mountain Car environments.

Together, these contributions extend conformal prediction to the infinite-horizon RL setting and offer a scalable, practical framework for uncertainty-aware policy evaluation.

\subsection{Related Work}

\paragraph{Risk-aware RL.} RL is a framework in which an agent interacts with an unknown environment to maximize its expected total reward. Due to the intrinsic randomness of the environment, even policies with high expected returns may occasionally yield very low rewards, which can be problematic in risk-sensitive applications such as healthcare \citep{liu2020reinforcement} or competitive games \citep{mnih2013playing}. For instance, in clinical decision-making, patient responses to treatments are stochastic, making it desirable to select actions that achieve high effectiveness while minimizing the likelihood of adverse effects. To address these concerns, risk-aware RL aims to learn policies that reduce the probability of low total rewards \citep{howard1972risk}, using a variety of risk measures including entropic or exponential utility \citep{fei2021risk,moharrami2025policy}, conditional value-at-risk \citep{rockafellar2000optimization,chow2015risk}, and coherent risk measures \citep{lam2022risk}.  

In parallel, safe RL and constrained Markov Decision Processes (MDPs) offer an alternative approach to managing uncertainty; a comprehensive survey of safe RL is provided in \cite{gu2024review}. Unlike risk-aware MDPs, these methods do not modify the optimality criteria; instead, risk aversion is enforced through constraints on rewards or risks \citep{chow2018risk}. While both risk-aware and safe RL approaches incorporate risk considerations into policy learning, they primarily focus on modeling risk preferences and generally do not provide formal uncertainty quantification for PIs.

\textbf{Distributional RL.} Distributional RL focuses on modeling the full return distribution rather than just its expectation. Pioneering work by \cite{bellemare2017distributional} introduces this paradigm, followed by quantile-based approaches such as Quantile Temporal Difference (QTD) learning \citep{dabney2018distributional, rowland2024analysis}, which approximates return distributions via quantile regression. These methods have led to practical advances in robotics, control, and decision-making under uncertainty \citep{bellemare2020autonomous, bodnar2019quantile, fawzi2022discovering, yang2019fully}. However, most DRL methods provide pointwise quantile estimates and lack formal statistical coverage guarantees, especially under model misspecification.

%In contrast, our framework builds on DRL to estimate return distributions but enhances it with conformal prediction to construct PIs with  asymptotic coverage  guarantees, even in infinite-horizon settings.

By integrating conformal prediction with DRL-based distribution estimation, our framework ensures asymptotic coverage for predictive intervals, even in challenging infinite-horizon settings. 

\textbf{Conformal Prediction for RL.} Conformal prediction offers distribution-free confidence intervals under exchangeable data \citep{vovk2005}. Extending it to RL is challenging due to the inherent temporal dependencies and evolving state distributions. Recent efforts have attempted to bridge this gap.  Early work such as \cite{dietterich2022conformal} applies conformal prediction to construct trajectory-level prediction intervals in finite-horizon MDPs. Building on this idea, \cite{foffano2023conformal} develop a weighted conformal prediction method for off-policy evaluation, using importance sampling weights to correct for distributional shifts between behavior and target policies. However, this approach suffers from the curse of horizon, as the importance weights accumulate multiplicatively over time, resulting in high variance in long-horizon settings. In parallel, \cite{zhang2023conformal} introduce the COPP algorithm for contextual bandits, which approximates exchangeability via pseudo-policies and trajectory subsampling; yet, its applicability is largely limited to short-horizon problems with finite discrete action spaces. \cite{zheng2024conformal} further analyze how temporal correlations in Markovian data affect the coverage and width of split conformal intervals. Finally, we note a growing line of work that applies adaptive conformal prediction to online safe RL settings \citep{sheng2024safe,zhou2025computationally}, which differs fundamentally from our  setting. 

 Despite these advances, existing methods largely focus on finite-horizon scenarios or on settings with limited state or action spaces. Prior conformal RL approaches typically handle distribution shifts between behavior and target policies using trajectory-level importance weighting, which becomes computationally inefficient as the trajectory horizon grows. In contrast, our work is the first to tackle infinite-horizon off-policy prediction in general RL settings with arbitrary state and action spaces using conformal prediction. By constructing stepwise pseudo-returns and leveraging experience replay, our method scales conformal prediction to infinite-horizon settings with standard RL data and remains effective even when only partial trajectory fragments are available.

\section{Problem Formulation}

We consider the standard RL framework \citep{bellemare2017distributional,kallus2020double,shi2022statistical}, where the environment is modeled as a  time-homogeneous MDP, as specified in the assumptions provided in the supplementary material.  Our goal is to construct distribution-free PIs for the return of a given policy in infinite-horizon settings under both on-policy and off-policy scenarios.

\paragraph{Data and Setup.} Let $\mathcal{D}=\{\zeta_{i}\}_{i=1}^{N}$ be a dataset of $N$ trajectories, each consisting of $T$ time steps. For simplicity, we assume trajectories have uniform length, but our method naturally extends to variable-length settings. Each trajectory $\zeta_i=\{(S_{it},A_{it},R_{it})\}_{t=0}^{T-1}$ consists of the state $S_{it}$, the action $A_{it}$ and the immediate reward $R_{it}$. These transitions are generated by a \textbf{behavior policy} $\pi_{b}$, such that $A_{it}\sim \pi_{b}(\cdot \mid S_{it})$ and evolve under a transition kernel $\mathcal{P}$ with $(R_{it}, S_{i,t+1}) \sim \mathcal{P}(\cdot \mid S_{it}, A_{it})$. In healthcare applications, each trajectory corresponds to a patient, with \(S_{it}\) representing clinical features, \(A_{it}\) the administered treatment, and \(R_{it}\) the resulting clinical response.

\paragraph{Objective.} Let $\pi$ be a \textbf{target policy} of interest. The return starting from the state $s$ is defined as 
$
G^{\pi}(s) = \sum_{t=0}^{\infty} \gamma^t R_t,
$ 
where $R_t$ is the reward at time $t$ under policy $\pi$ and $\gamma\in(0,1)$ is the discount factor. This return captures the long-term outcome of following policy $\pi$ from state $s$. Given a new test state $S_{\rm test}$, we aim to construct a prediction interval for $G^{\pi}(S_{\rm test})$ that achieves a user-specified coverage level $1-\alpha$. That is, we seek a set $C(S_{\rm test})$ such that: 
\begin{equation*}
\Pr(G^{\pi}(S_{\rm test})\in C(S_{\rm test}))\ge 1-\alpha.
\end{equation*}
In healthcare applications, \(G^{\pi}(S_{\rm test})\) represents the long-term treatment effect for a new patient under policy \(\pi\). The prediction interval thus provides a principled range of plausible outcomes for the patient, enabling informed decision-making before the policy is actually deployed in practice. In this paper, we consider two settings:
\begin{enumerate}[leftmargin=*]
    \item \textbf{On-Policy Setting:} the target policy $\pi$ is the same as the behavior policy $\pi_{b}$. This setting enables evaluation using in-distribution transitions, but still faces the challenges of infinite horizon and unobserved returns.
    \item \textbf{Off-Policy Setting:} the target policy $\pi$ differs from $\pi_b$. In this case, the data distribution differs from that under the  target policy, and appropriate corrections for distribution shift are necessary.
\end{enumerate}

\paragraph{Preliminaries of DRL.} The goal of DRL is to learn the distribution of returns $G^{\pi}(s)$ for each state $s$. Let $\eta^{\pi}(s)$ denote the the probability distribution of the random return. Numerous DRL methods exist for both on-policy and off-policy settings \citep{bellemare2023distributional}. In this paper, we adopt quantile temporal difference (QTD) learning for experiments, a prominent approach within DRL. QTD seeks to approximate the return distribution by
\(
\eta^{\pi}(s) \approx \frac1m  \sum_{i=1}^m  \delta_{\theta(s,i)}, 
%\label{QTD_aprox_dist}
\)
which is an equally-weighted mixture of Dirac deltas at locations $\theta(s,i)$. The aim is to have these particles approximate the $\tau_i=(2i-1)/(2m)$-th quantiles of $\eta^{\pi}(s)$ for $i=1,\ldots,m$. Like other temporal-difference methods, QTD updates its parameters $\{(\theta(s,i))_{i=1}^{m}\}$ using observed transitions $(S_{it},R_{it},S_{i,t+1})$. In continuous and high-dimensional state spaces, function approximation offers a powerful approach for modeling $\{(\theta(s,i))_{i=1}^{m}\}$ and generalizing across states.

\textbf{Limitations of DRL.} A naive approach to constructing PIs would be to take the empirical quantiles of $\eta^{\pi}(s)$, i.e. using $[\theta(s, L), \theta(s, U)]$, where $L =  \lfloor (m\alpha+1)/2\rfloor$ and $U = m+1-L$. However, such DRL-based quantile intervals, referred to as \textbf{DRL-QR}, can be unreliable in finite-sample settings and do not come with formal guarantees of asymptotic validity. For instance, \cite{bellemare2023distributional} show that the QTD algorithm converges to a limiting distribution in finite state and action spaces; yet this limiting distribution is not guaranteed to match the true return distribution, and thus the convergence provides no assurance that QTD-based prediction intervals are asymptotically valid. In continuous state and action spaces, distributional RL methods must rely on function approximation to estimate return distributions. The theoretical guarantees of these approaches consequently depend critically on the accuracy of the modeling assumptions, rendering them susceptible to potential model misspecification.  To address these limitations, we develop a conformal prediction framework that \textbf{wraps} around any return distribution estimator (such as QTD), correcting for model bias and enabling finite-sample statistical guarantees.

\section{Conformal Policy Prediction Beyond the Horizon}

We propose a novel conformal prediction (CP) framework that addresses the unique challenges of uncertainty quantification in infinite-horizon RL. Our approach combines three key innovations: (1) pseudo-returns that blend finite rollouts with learned distributional tails, (2) time-aware calibration addressing both temporal dependence and distribution shifts, and (3) replay-based weighted subsampling to restore exchangeability. 

\subsection{Overview of the Conformal Framework}

Our method follows the split conformal prediction paradigm, adapted to the RL setting. Given a dataset of transition tuples $\{(S_{it}, A_{it}, R_{it}, S_{i,t+1})\}$, we partition it into a \textbf{training set} $\mathcal{D}_{\text{tr}}$ , used to fit a predictive model for the return distribution, and a \textbf{calibration set} $\mathcal{D}_{\text{cal}}$, used to quantify predictive uncertainty. The overall pipeline consists of four key steps illustrated in Figure \ref{on_policy_flow_chart}:
\begin{enumerate}[leftmargin=*]
    \item Train a DRL model, such as QTD learning, on $\mathcal{D}_{\text{tr}}$ to construct a  return distribution  estimate $\hat{\eta}^\pi(s)$ and a value function  estimate $\hat{v}^\pi(s)$ under the target policy $\pi$.
    \item For each calibration state, construct \emph{pseudo-returns} by combining observed rewards with samples drawn from the estimated return distribution. The procedure for generating pseudo-returns is detailed in Section~\ref{sec:pseudo-return}.
    \item Compute nonconformity scores using the pseudo-returns in the calibration set, typically using the absolute deviation from the estimated value function: $V(s)=|\widetilde{G}^{\pi}(s)-\widehat{v}^{\pi}(s)|$, where $\widetilde{G}^\pi(s)$ denotes the pseudo-return.
    \item Apply conformal prediction to construct a prediction interval for a new test state $S_{\rm test}$, using weighted subsampling to adjust for distribution shifts and experience replay to approximate exchangeability by decorrelating transitions, detailed in Section~\ref{sec:calibration}. 
\end{enumerate}
The nonconformity score plays a central role in quantifying uncertainty and correcting for potential estimation bias. While our framework is compatible with more sophisticated nonconformity measures, such as those used in conformalized quantile regression  \citep{romano2019conformalized}, the double-quantile score \citep{foffano2023conformal}, and various others, we use the simple absolute-error score here for clarity and illustration.

\begin{figure}[h]
	\centering
	\includegraphics[width=0.8\columnwidth]{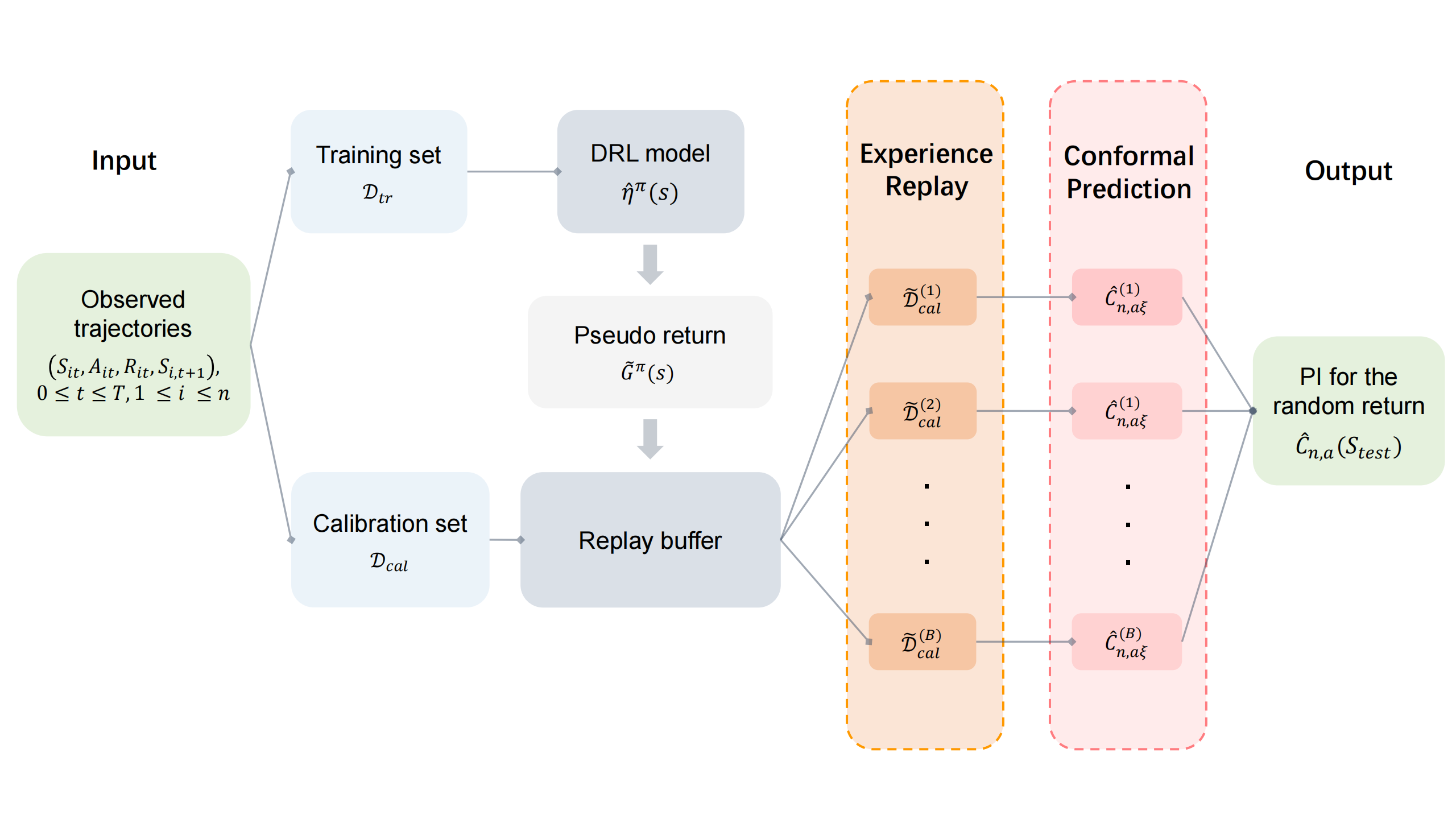}
	\caption{Pipeline of the proposed conformal policy prediction framework.}
	\label{on_policy_flow_chart}
\end{figure}

\subsection{Pseudo-Return Construction via Truncated Rollouts}\label{sec:pseudo-return}

A key challenge in infinite-horizon RL is that the true return $G^\pi(s)$ is unobservable in a finite-step trajectory, making it difficult to directly evaluate nonconformity scores for conformal prediction. To address this, we introduce a novel {pseudo-return construction} that inspired by $k$-step temporal difference (TD) learning. We reinterprete $k$-step TD learning through the lens of distributional inference. Specifically, for each calibration point $(S_{it}, A_{it}, R_{it}, S_{i,t+1})$, we define the $k$-step \emph{pseudo-return} as:
\begin{equation}\label{pseudo_return_on_policy}
    \tilde{G}^{(k)}(S_{it}) = \sum_{h=0}^{k-1} \gamma^h R_{i,t+h} + \gamma^k \tilde{G}^\pi(S_{i,t+k}),
\end{equation}
where the first term accumulates observed rewards under the behavior policy $\pi_b$, and the second term approximates the unobserved tail using a sample from the estimated return distribution $\hat{\eta}^\pi(S_{i,t+k})$.

\paragraph{Advantages.} {Pseudo-return construction} approximates the infinite-horizon return using a finite rollouts combined with a bootstrapped tail. \textbf{First}, this decomposition bridges model-based and model-free RL within the conformal inference framework. \textbf{Second}, the tail value is sampled from a learned return distribution, allowing seamless integration with DRL methods such as QTD or C51 \citep{bellemare2017distributional}. \textbf{Finally}, the rollout horizon $k$ offers a natural bias-variance trade-off: increasing $k$ incorporates more observed data, potentially reducing model bias but requiring longer rollouts; decreasing $k$ increases reliance on model predictions, offering faster calibration at the cost of higher bias.

\paragraph{On-policy setting.}  We detail the QTD learning procedure for DRL used in this paper, although any DRL estimation method can be integrated into our framework. In the on-policy case, QTD estimates the return distribution conditioned on the initial state, $\hat{\eta}^{\pi}(s)$, via the iterative update
\[
\theta(s,i) \leftarrow \theta(s,i) + \rho\cdot \frac{1}{m} \sum_{j=1}^{m} \left[ \tau_i - I(r + \gamma \theta(s',j) - \theta(s,i) < 0) \right],
\]
where $\theta(s,i)$ denotes the $\tau_i$-th quantile of $\hat{\eta}^{\pi}(s)$,   $(s,a,r,s')$ is sampled under the behavior policy $\pi$, which coincides with the target policy in the on-policy setting, and $\rho$ is a learning rate.  

\paragraph{Off-policy setting.}  Extending QTD to the off-policy setting requires careful modifications to account for distributional shifts between $\pi_b$ and $\pi$. We first define the return starting from a state-action pair as $G^{\pi}(s,a)=\sum_{t=0}^{\infty}\gamma^t R_t$,  where the agent takes action $a$ in state $s$ and follows policy $\pi$ thereafter. The distribution of this return is denoted by $\eta^{\pi}(s,a)$. The goal of QTD is to estimate the quantile functions of $\eta^{\pi}(s,a)$. The iterative update for the $\tau_i$-th quantile $\theta(s,a,i)$ is given by
\[
\theta(s,a,i) \leftarrow \theta(s,a,i) + \rho\cdot \frac{1}{m} \sum_{j=1}^{m} \left[ \tau_i - I(r + \gamma \theta(s',a',j) - \theta(s,a,i) < 0) \right],
\]
where $\theta(s,a,i)$ is the $\tau_i$-th quantile of $\hat{\eta}^{\pi}(s,a)$, $(s,a,r,s')$ is sampled from the behavior policy $\pi_b$, and $a'$ is drawn from the target policy $\pi$. The result is marginalized over the action space according to $\pi$: $ \widehat{\eta}^{\pi}(s)=\sum_{a}\pi(a|s)\widehat{\eta}^{\pi}(s,a)$. This modification is necessary to correct for the action distribution mismatch between behavior and target policies. For further details on distributional RL in off-policy evaluation, see \cite{qi2025distributional,hong2024distributional}.

\subsection{Time-Aware Calibration via Experience Replay and Weighted Subsampling}\label{sec:calibration}

A core challenge in applying CP to RL lies in the violation of its key assumption: \emph{exchangeability} between the calibration and test data. In RL, this is broken due to (i) temporal dependencies across transitions and (ii) distribution shifts in the state space both over time and across policies. To address these challenges, we introduce a two-pronged calibration strategy through \emph{experience replay-based sampling} to decorrelate temporally linked transitions and \emph{time-aware importance weighting} to correct for dynamic policy-dependent distributional shifts.

\paragraph{Experience Replay.}
Temporal dependence between transitions in RL makes the direct application of conformal prediction invalid. To mitigate this, we draw inspiration from deep RL techniques and treat the calibration set as a \emph{replay buffer}, storing transition tuples $(S_{it}, A_{it}, R_{it}, S_{i,t+1})$. We then apply random subsampling from this buffer to construct approximately i.i.d. calibration samples \citep{fan2020theoretical}. This technique mirrors the prioritized or uniform experience replay used in deep Q-learning, effectively decorrelating transitions \citep{schaul2015prioritized}. For the construction of $k$-step pseudo-returns, we store extended tuples of the form $\{(S_{it}, A_{it}, R_{it}, \ldots, S_{i,t+k})\}$.

\paragraph{Weighted Subsampling (WS).} Instead of adopting weighted conformal prediction (WCP) \citep{tibshirani2019conformal}, which is commonly used to correct for covariate shifts, we employ a \emph{sampling-based} strategy. Specifically, we perform weighted  subsampling from the calibration buffer based on estimated importance weights, producing a recalibrated set of approximately exchangeable samples tailored to the target distribution. The importance weights differ depending on whether the setting is on-policy or off-policy:

\begin{enumerate}[leftmargin=*]
\item \textbf{On-Policy Setting.} Here, the distribution shift stems from time-indexed variation in state visitation. We define the importance weight as
\begin{equation}\label{weight_function_on_policy}
    w_{\text{on}}(s) = \frac{d\mathcal{P}_0(s)}{d\mathcal{P}_{\text{cal}}(s)}=\frac{P(\delta=1\mid s)}{P(\delta=0\mid s)}\frac{P(\delta=0)}{P(\delta=1)}  \propto \frac{P(\delta=1\mid s)}{P(\delta=0\mid s)},
\end{equation}
where $\mathcal{P}_0$ is the  probability  distribution of test states,  $\mathcal{P}_{\text{cal}}$ is the marginal  probability  distribution over calibration states, and $\delta$ is an indicator variable, where $\delta=0$ denotes that $s$ belongs to the calibration set, and $\delta=1$  indicates that $s$ is in the test set. The second equality in Eq.~\eqref{weight_function_on_policy} follows from Bayes’ rule, expressing the likelihood ratio as a ratio of classifier probabilities \citep{friedman2004multivariate,qiu2023prediction}. In practice, $w_{\text{on}}(s)$ can be estimated using standard propensity scoring or density ratio estimation methods. In simulations, we employ logistic regression for this purpose. 

\item \textbf{Off-Policy Setting.} In this case, both temporal drift and policy mismatch must be corrected. We define the importance weight over a $k$-step trajectory segment as
\begin{equation}
    w_{\text{off}}(s_0, a_0, \ldots, s_k) \propto \frac{d\mathcal{P}_0(s_0)}{d\mathcal{P}_{\text{cal}}(s_0)} \prod_{h=0}^{k-1} \frac{\pi(a_h \mid s_h)}{\pi_b(a_h \mid s_h)}.
\end{equation}
This formulation adjusts for discrepancies in both state visitation and action selection between the behavior and target policies. This ratio can also be estimated using propensity scoring techniques.
\end{enumerate}

To reduce the variance in PIs caused by subsampling randomness, we repeat the process $B$ times and aggregate the intervals. This technique draws from recent work in conformal prediction under distribution shift \citep{zhang2023conformal} and improves both coverage stability and efficiency. {The complete algorithm for the on-policy setting is in Algorithm \ref{alg:CP_on_policy}, while the off-policy version is deferred to the supplementary material to save space.}

\paragraph{Why WS Works.}  In the off-policy setting, let \( S_{\text{test}} := S_{\text{test},0} \) denote a test state drawn from the marginal distribution \( \mathcal{P}_0(s) \), and consider the joint distribution:
\begin{equation*}
(S_{\text{test},0}, A_{\text{test},0}, R_{\text{test},0}, \ldots, S_{\text{test},k}, G^{\pi}(S_{\text{test},0})) \sim \mathcal{P}_{0}^{\text{off}}(s_0, a_0, r_0, \ldots, s_k, G).
\end{equation*}
Similarly, let \( \mathcal{P}_{\text{cal}}^{\text{off}} \) denote the joint distribution of rollout segments in the calibration set:
\begin{equation*}
(S_{it}, A_{it}, R_{it}, \ldots, S_{i,t+k}, G^{\pi}(S_{it})) \sim \mathcal{P}_{\text{cal}}^{\text{off}}(s_0, a_0, r_0, \ldots, s_k, G).
\end{equation*}
The two distributions are related through the importance weight \( w_{\text{off}} \), such that:
\begin{align}\label{weight_function_off_policy}
d\mathcal{P}_{0}^{\text{off}}(s_0, a_0, r_0, \ldots, s_k, G) = w_{\text{off}}(s_0, a_0,  \ldots, s_k) \, d\mathcal{P}_{\text{cal}}^{\text{off}}(s_0, a_0, r_0, \ldots, s_k, G).
\end{align}
This identity shows that sampling from the calibration distribution according to the importance weights \( w_{\text{off}} \) produces samples that approximate the test-time distribution \( \mathcal{P}_{0}^{\text{off}} \). By reweighting the calibration set in this way, we recover approximate exchangeability between the calibration and test samples, thereby restoring the validity of conformal prediction in the presence of both temporal and policy-induced distribution shifts.

\paragraph{Why Not Use WCP.}
 Weighted conformal prediction (WCP) typically assumes access to the full set of test-time covariates. In contrast, our setting only observes the initial state $S_{\text{test},0}$ at test time, while subsequent states $S_{\text{test},1}, S_{\text{test},2}, \ldots, S_{\text{test},k}$ remain unobserved. The WCP weight defined in Eq.~12 of \cite{foffano2023conformal} involves marginalizing over entire trajectories, which are unobserved. Although \cite{foffano2023conformal} further propose an optimization-based approximation (Eq.~14), this approach introduces additional model assumptions and tends to exhibit high variance, especially in long-horizon settings, limiting their practical applicability in our context. On the other hand, while one could adopt   more elaborate designs such as that of \cite{zhang2023conformal} tailored for sequential decision-making, our weighted subsampling scheme offers a significantly simpler and more practical alternative, especially when only the initial states of test trajectories are observed.

\begin{center}
\begin{algorithm}[ht] \label{alg:CP_on_policy}
	\caption{ {\it  CP for Infinite Horizon On-policy Evaluation  } }
	\KwData{  
		$\mathcal{D} = \{ (S_{it},A_{it},R_{it},S_{i,t+1}) : 1 \le i \le N, 1 \le t \le T \}$ and a test state $S_{\rm test}$.
	}
	
	\KwIn{ $1-\alpha$, target coverage level; 
		$\mathcal{A}$, an on-policy distributional RL algorithm; 
		$\mathcal{W}$, a density ratio estimation algorithm; 
		$k$, step width;
		$B$, resampling number;
		$l$, subsample size;  
		$\xi$, multiple subsampling parameter
	}

	\KwOut{Prediction interval for $G^\pi(S_{\rm test})$}

	Split the data: 
    $\mathcal{D} 
    = \mathcal{D}_{\rm tr} \bigcup \mathcal{D}_{\rm cal}$
    where 
    $ \mathcal{D}_{\rm tr} = \left\{(S_{it},A_{it},R_{it},S_{i,t+1}) : (i,t) \in \mathcal{I}_{\rm tr} \right\}$
    and
    $\mathcal{D}_{\rm cal} = \left\{ (S_{it},A_{it},R_{it},\ldots,S_{i,t+k}) : (i,t) \in \mathcal{I}_{\rm cal} \right\}$.
    Here, $\mathcal{I}_{{\rm tr}}$ and $\mathcal{I}_{{\rm cal}}$ denote the indices of transitions in the training and calibration datasets, respectively. 
    
	Train a conditional return model $\hat\eta^\pi(s)$ using $\mathcal{A}$ based on $\mathcal{D}_{\rm tr}$.
	
	Obtain the value function estimator $\hat v^\pi(s)$, the expectation of $\hat\eta^\pi(s)$.

	Obtain $\hat w_{\rm on}(\bs)$ as an estimator of the density ratio (\ref{weight_function_on_policy}) based on 
	$\left\{ S_{i0}: (i,0) \in \mathcal{I}_{\rm tr} \right\}$ and
	$\left\{ S_{it}: (i,t) \in \mathcal{I}_{\rm tr} \right\}$ using $\mathcal{W}$.

	\For{$b=1:B$} {
		\bit
		\item
		Sample $l$ data tuples
        $\{(S_{it}, {A}_{i,t}, {R}_{i,t}, \ldots,S_{i,t+k}) : (i,t) \in \mathcal{I}_{\rm cal}^{(b)}\}$ from $\mathcal{D}_{\rm cal}$ according to the importance weight $\hat w_{\rm on}(S_{it})$.

        \item
        Calculate pseudo return  (\ref{pseudo_return_on_policy}) and obtain  $\widetilde{\mathcal{D}}_{\rm cal}^{(b)}
        := \{(S_{it}, \widetilde{G}_{it}^{(k)}) : (i,t) \in \mathcal{I}_{\rm cal}^{(b)}\}$.

		\item
		Calculate the nonconformity scores: 
		$\{
		V_{it}
		:=  | \widetilde{G}_{it}^{(k)} - \hat v^\pi(S_{it}) |  : (i,t) \in \mathcal{I}_{\rm cal}^{(b)}\}
		\}$.
		
		\item 
		Obtain $\hat{q}_{1-\alpha \xi}^{(b)} $, the $\lceil l (1-\alpha \xi)\rceil$-th smallest value of $\{ V_{it} : (i,t) \in \mathcal{I}_{\rm cal}^{(b)} \}$.
		
		\item Obtain 
		$
		\widehat C_{N, \alpha \xi}^{(b)}(S_{\rm test}) 
		= \hat v^\pi(S_{\rm test}) \pm \hat q_{1- \alpha \xi}^{(b)}. 
		$
		
		\eit
	}

	\KwResult{
		A conformal predictive region for $G^\pi(S_{\rm test})$ with a coverage rate of $1-\alpha$  is
		\ba
		\widehat{C}_{N,\alpha}^{\rm on} (\rm S_{\rm test})
		&=&
		\left\{ G: \frac1B \sum_{b=1}^B 
		I\left\{ G \in \widehat C_{N,\alpha \xi}^{(b)}(S_{\rm test})  \right\} \ge 1-\xi 
		\right\}.
		\label{PI_on_policy}
		\ea
	}
\end{algorithm}
\end{center}

\section{Theoretical Results}
In this section, we provide statistical guarantees for the PIs constructed by our method. Standard CP yields marginal coverage at level $1-\alpha$  under the assumption of exchangeability. However, in practice, distribution shifts violate this assumption, leading to a gap between the nominal level $1-\alpha$ and the actual coverage. Previous studies have bounded this gap using total variation distance, which fails to capture how different choices of $k$ in $k$-step rollouts affect the coverage gap. To address this, we propose a tighter upper bound on the coverage gap based on the Wasserstein distance, leveraging a recent theoretical result from \cite{xu2025wasserstein}. { Let  $\mu$ and $\nu$ be two probability measures on the real space $\mathbb{R}$. 
For any $p>0$, the $p$-Wasserstein distance between $\mu$ and $\nu$ is defined as
\(
W_p(\mu, \nu) :=  \inf_{\kappa \in \Gamma(\mu,\nu)}\{\int_{\mathbb{R}\times \mathbb{R}} |x - y|^p \kappa(dx, dy)\}^{1/p},   
\)
where $\Gamma(\mu,\nu)$ denotes the set of all couplings with marginals $\mu$ and $\nu$. 
}

 Let $n $ be the cardinality of the calibration set $\mathcal{D}_{\rm cal}$, 
and $\hat\eta^{\pi}(s)$ denote an estimate of the return distribution $\eta^{\pi}(s)$  under the target policy \(\pi\).
We take $\mathcal{S}$ to be the state space and define 
$\bar{W}_1(\eta^\pi, \hat{\eta}^\pi) := \sup_{s \in \mathcal{S}} W_1(\eta^\pi(s), \hat{\eta}^\pi(s))$. 
Let $\widehat{w}_{\rm on}(s)$ be an estimate of the on-policy importance weight defined in \eqref{weight_function_on_policy}, and let $\widehat C_{N,\alpha}^{\rm on}(\cdot)$ be the prediction interval produced by Algorithm~\ref{alg:CP_on_policy}.  
The following theorem establishes an asymptotic lower bound on the coverage in the on-policy setting.

\begin{condition}\label{condition_on} \rm
(i) The return distribution $\eta^\pi(s)$ has a Lebesgue density bounded by $L$ for all $s \in \mathcal{S}$.
(ii) $ \mathbb{E}[\hat{w}_{\mathrm{on}}(S_{it})| \mathcal{D}_{\mathrm{tr}}] < \infty$ {and}  $\mathbb{E}[w_{\mathrm{on}}(S_{it})] < \infty$ {for all } $0 \le t \le T-k.$ 
    
\end{condition}
\begin{theorem}[\bf On-Policy Coverage Guarantee]\label{thm:on-policy}
Assume   Condition~\ref{condition_on}, and redefine $\hat w_{\rm on}(s)$ as 
	$\hat w_{\rm on}(s)/\frac{1}{T-k+1}\sum_{t=0}^{T-k}\e[\hat w_{\rm on}(S_{it}) | \mathcal{D}_{\rm tr}]$ so that $\frac{1}{T-k+1} \sum_{t=0}^{T-k}\e[ \hat w_{\rm on}(S_{it})| \mathcal{D}_{\rm tr}] = 1$.
Then %the prediction interval $\widehat{C}_{N,\alpha}^{\mathrm{on}}(S_{\rm test})$ satisfies
\begin{align*}
&\lim_{n \to \infty} \Pr\left( G^{\pi}(S_{\rm test}) \in \widehat{C}_{N,\alpha}^{\mathrm{on}}(S_{\rm test}) \right) \ge  1 - \alpha - \Lambda(\widehat{w}_{\rm on},\widehat{\eta}^\pi),~\text{where}\\
&\Lambda(\widehat{w}_{\rm on},\widehat{\eta}^\pi) = \frac{1}{2(T - k + 1)} \sum_{t=0}^{T-k} \mathbb{E} \left[ \left| \hat{w}_{\mathrm{on}}(S_{it}) - w_{\mathrm{on}}(S_{it}) \right| \right]  + \sqrt{2 L \gamma^k \, \mathbb{E} \left[ \bar{W}_1(\eta^\pi, \hat{\eta}^\pi) \right]}. 
\end{align*}
 
\end{theorem}

Theorem \ref{thm:on-policy} shows that the deviation from nominal coverage depends on two main factors: (i) the estimation error in the importance weights, which arises due to the distribution shift, and (ii) the approximation error in the return distribution $\widehat{\eta}^\pi(s)$, measured by the Wasserstein distance. Notably, the second term decays with the truncation step $k$ at a rate proportional to $\gamma^k$. When the approximation error in the return distribution $\widehat{\eta}^\pi(s)$ is large, choosing a larger $k$ can help reduce the deviation from nominal coverage by relying more on observed rewards.  However, this introduces a trade-off: if $k$ is too large, it becomes difficult to accurately estimate the off-policy weights, especially under substantial distributional shifts. In this case, the method effectively reduces to a Monte Carlo estimator that relies on full trajectories, resulting in the high variance we aim to avoid.

Next, we establish an asymptotic lower bound on the coverage of the PI in the off-policy setting. Let   \(\hat{w}_{\mathrm{off}}(\cdot)\) be an estimate of the importance weight \(w_{\mathrm{off}}(\cdot)\) as defined in~\eqref{weight_function_off_policy}. Let \(\widehat{C}_{N,\alpha}^{\mathrm{off}}(\cdot)\) denote the conformal interval produced by Algorithm 1 in the supplementary material.

\begin{condition}\label{condition_off} \rm
    (i) The return distribution \(\eta^\pi(s)\) has a 
Lebesgue density bounded by \(L\) for all $s \in \mathcal{S}$.
(ii) $\mathbb{E} [\hat{w}_{\rm off}(\mathcal{H}_{t:t+k})|\mathcal{D}_{\rm tr}] < \infty$, $\mathbb{E} [w_{\rm off}(\mathcal{H}_{t:t+k})]< \infty$
for all $0 \le t \le T - k$, 
where \(\mathcal{H}_{t:t+k} := (S_{t}, A_{t}, \ldots, S_{t+k})\) denotes 
the local trajectory segment following policy $\pi_b$, independent of $\mathcal{D}_{\rm tr}$.
(iii) (overlapping) $\pi_{b}(a|s)$  is uniformly bounded away from 0 for any $a,s$. 
\end{condition}
\begin{theorem}[\bf Off-Policy Coverage Guarantee]\label{thm:off-policy} 
Assume   Condition~\ref{condition_off}, and redefine 
$\hat w_{\rm off}(s_0,a_0.\ldots,s_{k+1})$ 
as $\hat w_{\rm off}(s_0,a_0.\ldots,s_{k+1})/\frac{1}{T-k+1}\sum_{t=0}^{T-k}\e[\hat w_{\rm off}(\mathcal{H}_{t:t+k})| \mathcal{D}_{\rm tr}]$ 
so that $\frac{1}{T-k+1} \sum_{t=0}^{T-k}\e[ \hat w_{\rm off}(\mathcal{H}_{t:t+k})| \mathcal{D}_{\rm tr}] = 1$.
Then we have
\begin{align*}
&\lim_{n \to \infty} \Pr\left( G^{\pi}(S_{\rm test}) \in 
\widehat{C}_{N,\alpha}^{\mathrm{off}}(S_{\rm test}) \right)\ge 1 - \alpha-\Lambda(\hat{w}_{\rm off},\hat{\eta}^\pi),~\text{where} \\
&\Lambda(\widehat{w}_{\rm off},\hat{\eta}^\pi)= \frac{1}{2(T - k + 1)} \sum_{t=0}^{T - k} 
\mathbb{E}  \left[ 
\left| \hat{w}_{\mathrm{off}}(\mathcal{H}_{t:t+k}) 
- w_{\mathrm{off}}(\mathcal{H}_{t:t+k}) \right| \right] + \sqrt{2 L \gamma^k \, 
\mathbb{E}  \left[ 
\bar{W}_1(\eta^\pi, \hat{\eta}^\pi) \right]}.
\end{align*}
\end{theorem}

Theorem~\ref{thm:off-policy} shows that the coverage deviation has the same form as in the on-policy case (Theorem~\ref{thm:on-policy}). The main difference is the additional estimation error in the importance weights $\widehat{w}_{\rm off}$, which arises from evaluating a different target policy. 

\paragraph{Remark.} For continuous return distributions, the bounded Lebesgue density assumption is mild and typically satisfied in practice. It holds for many commonly-used distributions, including the Gaussian, exponential, and Gamma distributions with shape parameter no less than 1. For example, in Examples 1 and 2 of our experiments, the return distributions can be readily verified to satisfy this condition. In contrast, this assumption does not apply to discrete return distributions, as discrete random variables are not absolutely continuous with respect to the Lebesgue measure. Hence, the bounded density condition is neither required nor meaningful for discrete returns, as in Example 3 of our experiments. 

\section{Experiments}

In this section, we conduct simulation studies to investigate the empirical performance of our proposed methods. In particular, we focus on the following two examples:

\paragraph{Example 1: two-state MDP (Example 3 of  \cite{rowland2024analysis})}
The state space of the environment is discrete with two possible values: $x_1$ and $x_2$. The agent transfers from a current state to a different state with a certain probability determined by the policy and the discount factor is $\gamma = 0.8$.
The reward obtained when transitioning from state $x_1$ is distributed as $N(2, 1)$, and the reward obtained when transitioning from state $x_2$ is distributed as $N (1, 1)$.

\paragraph{Example 2: continuous state (Scenario B of \cite{shi2022statistical})}

The action is binary and   $S_{t+1}=(S_{t+1,1}, S_{t+1,2})$, where $S_{t+1,1}=3(2A_{t}-1)S_{t,1}/4+z_{t,1}$, $S_{t+1,2}=3(1-2A_t)S_{t,2}+z_{t,2}$, $z_{t}=(z_{t,1},z_{t,2})$, for $t \geq 0$, $\{z_t\}_{t\geq0}{\operatorname*{\sim}}N(0_2,I_2/4)$ are i.i.d. and $S_{0}\sim N(0_2,I_2)$. The immediate reward $R_{t}=2S_{t+1,1}+S_{t+1,2}-(2A_t-1)/4$. The discount factor is $\gamma = 0.8$.

For each example, we consider both an on-policy setting and an off-policy setting:

\begin{itemize}[leftmargin=*]
\item
In Example 1, when there is no policy shift, the probabilities of transferring from $x_1$ to $x_2$ and $x_2$ to $x_1$ are 0.4 and 0.8, respectively; when there exists a policy shift, the training data has the same transition dynamics as in the on-policy setting, while the test agent transitions from $x_1$ to $x_2$ with probability 0.5 and from $x_2$ to $x_1$ with probability 0.7. 

\item 
In Example 2, when there is no policy shift, both the observed data and the test agents satisfy $\Pr(A_{t}=1|S_{t})=0.5 {\rm sigmoid} (S_{t,1})+0.5  {\rm sigmoid} (S_{t,2})$; when there exists a policy shift, the observed data follows the same policy as in the on-policy setting while the test data satisfies $\Pr(A_{t}=1|S_{t})= 0.6 {\rm sigmoid} (S_{t,1})+0.4 {\rm sigmoid} (S_{t,2})$. 
\end{itemize}

\paragraph{Implementation details.}
The sample size is fixed to $N = 400$ for Example 1 and $N=200$ for Example 2, with each trajectory consisting of $T=30$ stages. For Example 1, we approximate the return distribution using 20 conditional quantiles estimated by QTD. In Example 2, where the state space is continuous, we use 30 conditional quantiles estimated by QTD and model the conditional quantile functions with a neural network. The detailed architecture of the neural network is provided in the supplementary material. We evaluate the performance of the proposed method with step sizes $k=1,\ldots,5$, and set the number of intervals $B=50$. For each simulation, we generate 310 test points from the target policy to evaluate the converge probability.  In the supplementary material, we include simulation results for Example 1 to examine the impact of $\xi$ and $k$, a comparison with \cite{foffano2023conformal} based on the same example, and an extension of Example 1 to a high-dimensional setting.

\paragraph{Benchmark and Results.}
We compare our method with the quantile region given by the learned QTD model (DRL-QR). Since the DRL algorithm directly learns the return distribution $\eta^\pi(S) := \mathcal{P}(G^\pi|S)$ by $\widehat\eta^\pi(S)$, a quantile  region for the test instance $S_{\rm test}$ can be constructed as 
$[\widehat Q_{a/2}(S_{\rm test}),\widehat Q_{1-a/2}(S_{\rm test})]$, where $\widehat Q_{\tilde a}(S_{\rm test})$ is the $\tilde a$-th quantile of $\widehat\eta^\pi(S_{\rm test})$. Figure \ref{fig:ex1-2} presents boxplots based on 50 independent repetitions. It shows that our method consistently achieves near-nominal $90\%$  coverage across various $k$-step pseudo-returns in both on-policy and off-policy settings. In contrast, the DRL-QR baseline suffers from undercoverage due to model bias in the estimated return distribution. This highlights the effectiveness of our conformal framework in correcting such bias and ensuring valid uncertainty quantification. We also observe that the average interval length increases with larger $k$, reflecting the higher variance introduced by longer truncation horizons.

\begin{figure}[ht]
	\centering
	\subfigure[Coverage probability]{
		\begin{minipage}[]{\columnwidth}
			\centering
			\includegraphics[width=.245\columnwidth]{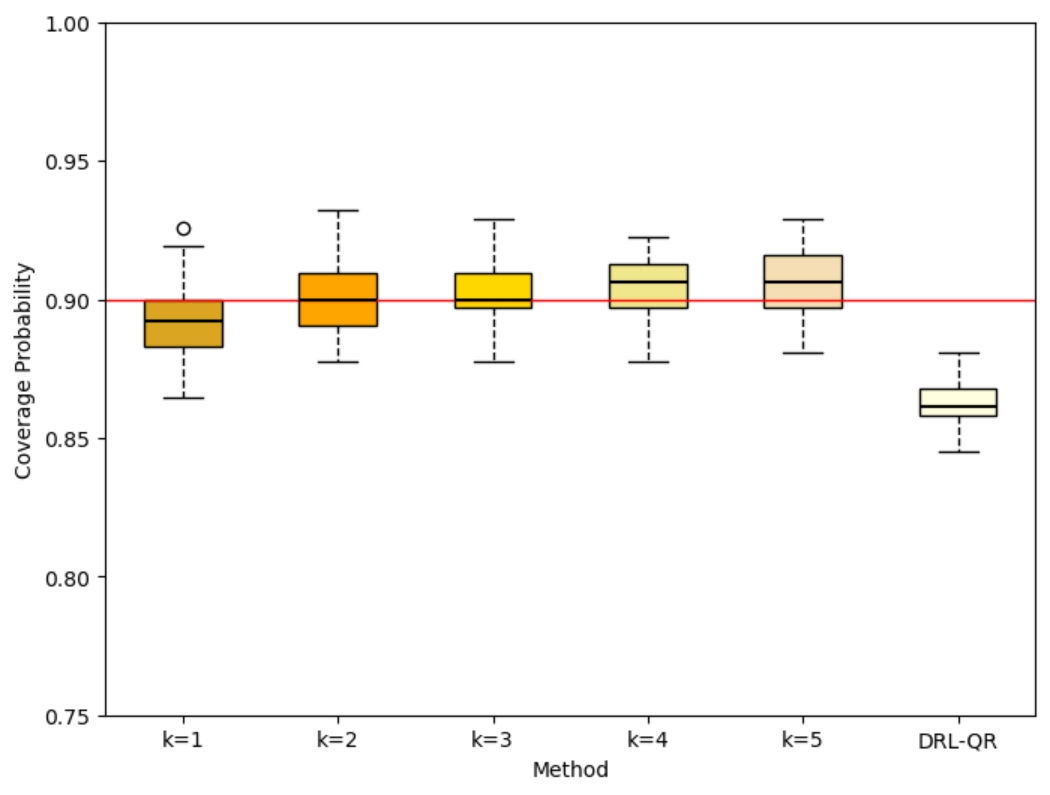}
			\includegraphics[width=.245\columnwidth]{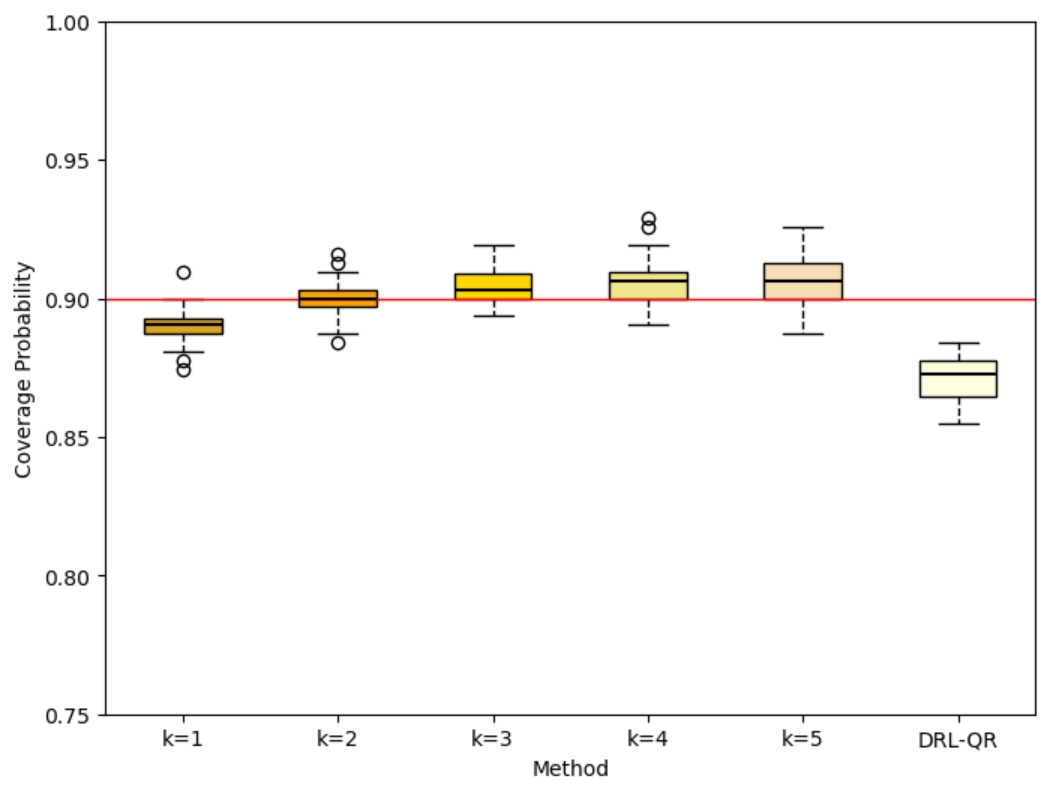}
            \includegraphics[width=.245\columnwidth]{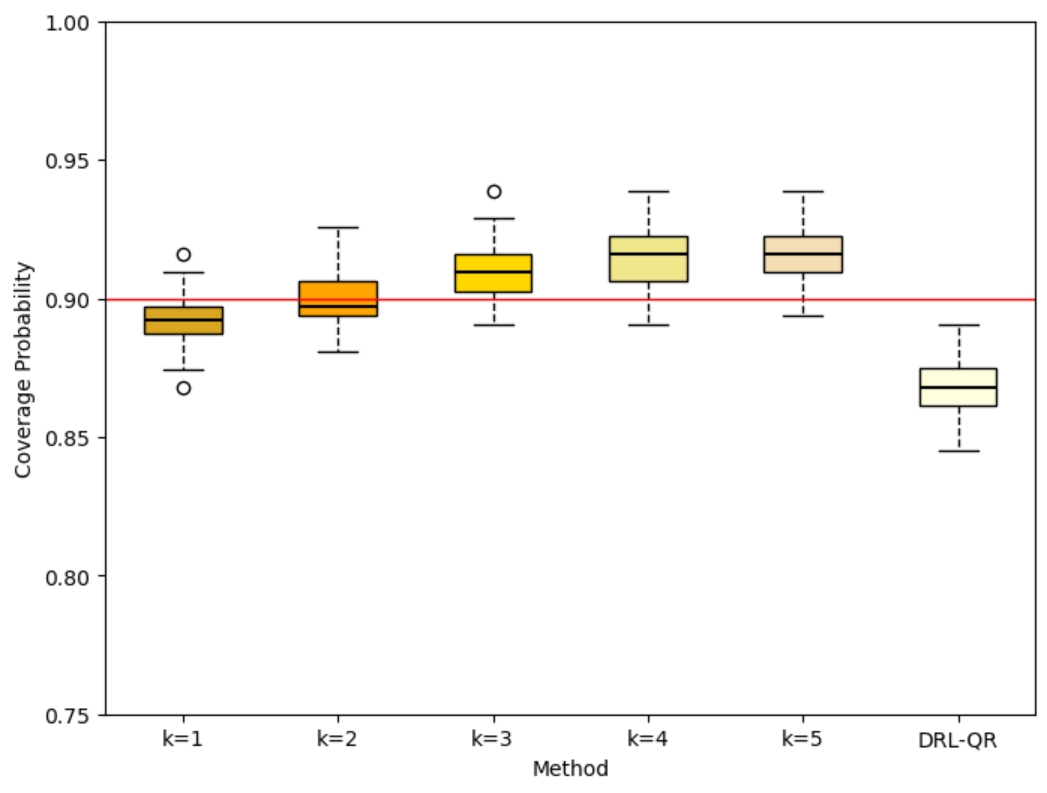}
			\includegraphics[width=.245\columnwidth]{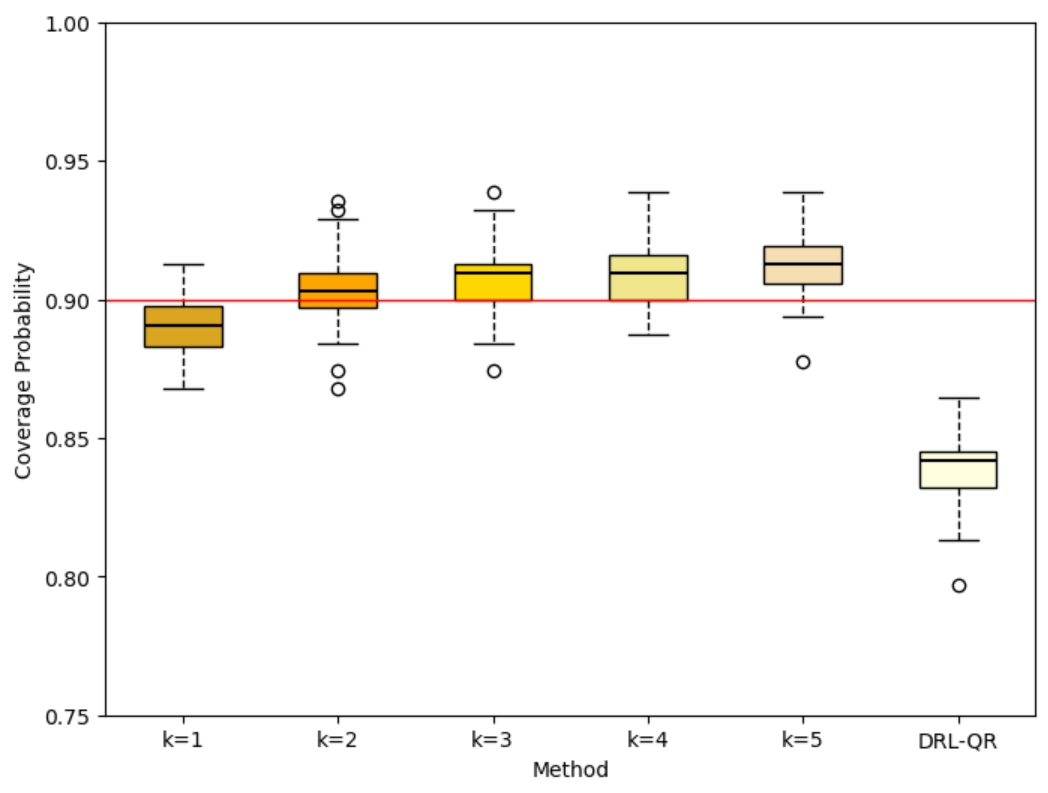}
		\end{minipage}
	}
	\subfigure[Average length]{
		\begin{minipage}[]{\columnwidth}
			\centering
			\includegraphics[width=.245\columnwidth]{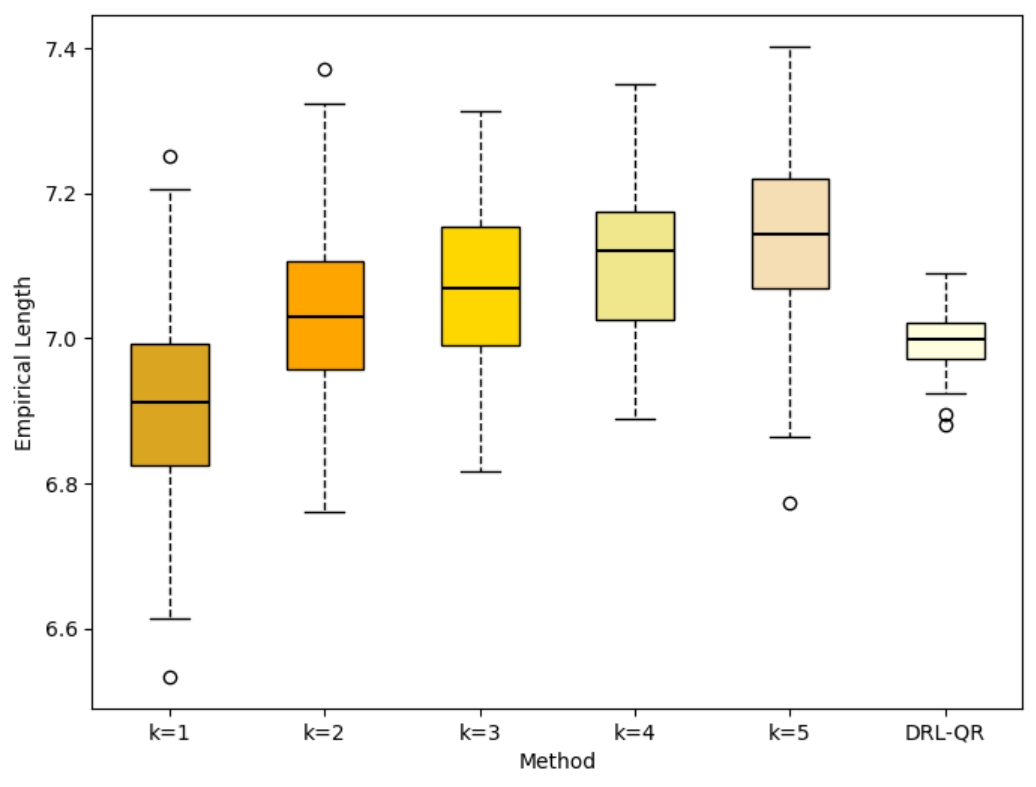}
			\includegraphics[width=.245\columnwidth]{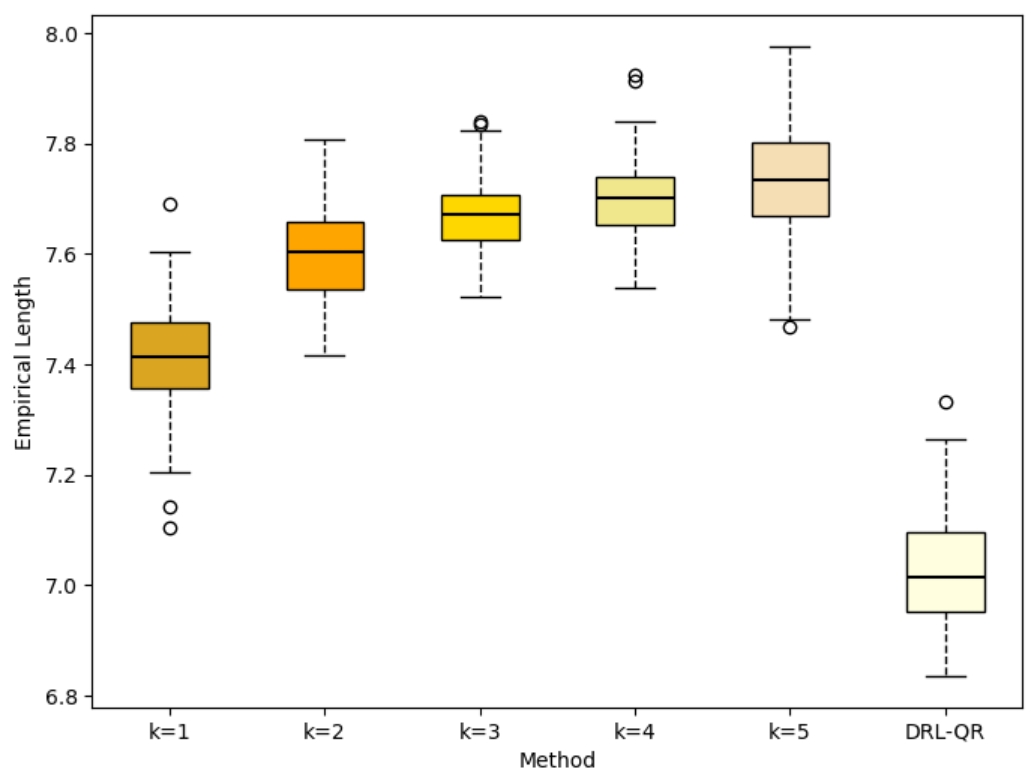}
            \includegraphics[width=.245\columnwidth]{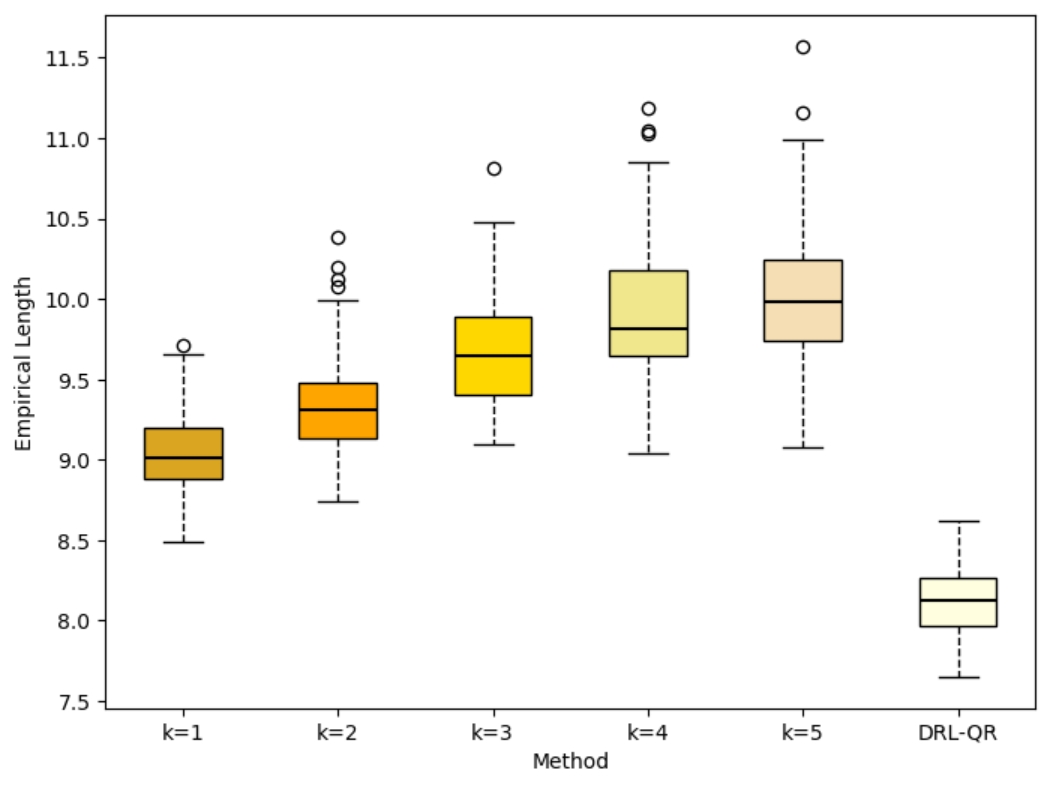}
			\includegraphics[width=.245\columnwidth]{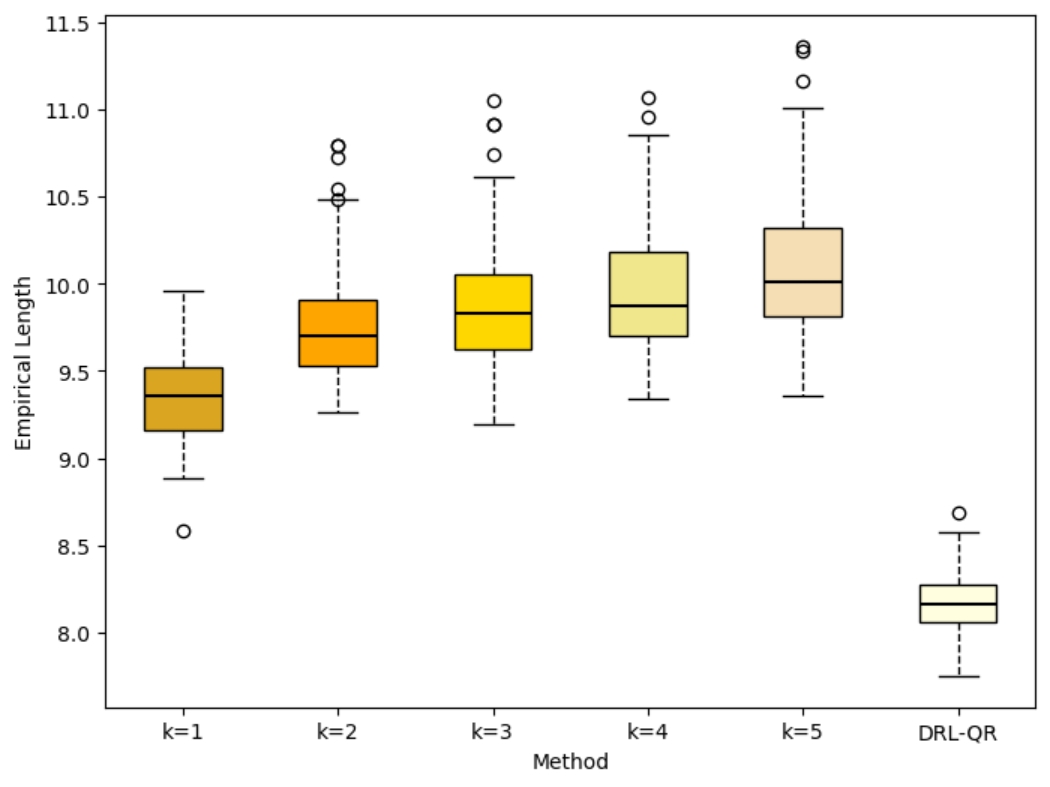}
		\end{minipage}
	}
    \caption{{Coverage probability and average interval length at the 90\% level for the proposed method with $k$-step pseudo-returns ($k = 1,\ldots,5$, from left to right) and DRL-QR (rightmost), under on-policy and off-policy settings in Example 1 (columns 1-2) and Example 2 (columns 3-4).}}\label{fig:ex1-2}
\end{figure}

\paragraph{Example 3: Mountain Car (adapted from \cite{kallus2020double})}  
We generate the dataset using a behavior policy defined as \(\pi_b = a\pi_{Q} + (1 - a)\pi_{U}\), where \(\pi_Q\) is a policy trained via Q-learning, \(\pi_U\) is a uniformly random policy, and \(a = 0.3\). The target policy is constructed similarly with \(a = 0.2\), reflecting an off-policy setting. To conserve space, implementation details and results are provided in the supplementary material. As a benchmark, we apply kernel density estimation (KDE) to approximate the return distribution from Monte Carlo rollouts and construct baseline prediction intervals using quantiles (KDE-QR). As shown in Figure 1 of the supplementary material, our method effectively corrects the model bias in KDE and achieves near-nominal 90\% coverage, highlighting the robustness of the proposed CP framework in a complex, continuous control task.

\section{Conclusion}
In this paper, we propose a novel CP framework for infinite-horizon policy evaluation with  asymptotic  coverage guarantees. By constructing $k$-step pseudo-returns, our method balances predictive accuracy and statistical efficiency, addressing key challenges in long-horizon evaluation. This formulation enables the construction of valid PIs without relying on full trajectory rollouts. Although the choice of $k$ remains underexplored, we suggest practical remedies such as evaluating stability across multiple 
$k$ values (e.g., $k=1,\ldots,5$) or aggregating PIs  across different $k$.  Since these intervals are correlated, aggregation is nontrivial. A promising direction is to construct a unified prediction region by combining the corresponding p-values, leveraging the connection between prediction intervals and hypothesis testing. Methods such as the Cauchy Combination Test \citep{liu2020cauchy}, which are robust to arbitrary dependencies, offer a viable approach. Moreover, extending our framework to policy optimization represents an exciting avenue for future work and could further broaden the applicability of conformal prediction in RL.

\section*{Acknowledgement} Zhang's research was supported by the National Natural Science Foundation of China (Grant No. 12471280) and the Shanghai Municipal Education Commission (Grant No. 2024AI01002). Liu's research was supported by the National Natural Science Foundation of China (Grant No. 12571283).

\bibliographystyle{plain}
\bibliography{main}

\newpage
\section*{NeurIPS Paper Checklist}

\begin{enumerate}

\item {\bf Claims}
    \item[] Question: Do the main claims made in the abstract and introduction accurately reflect the paper's contributions and scope?
    \item[] Answer: \answerYes{}.
    \item[] Justification: The abstract and introduction clearly state the contributions: we develop a novel conformal prediction (CP) framework to construct prediction intervals (PIs) for reinforcement learning (RL) settings, addressing key challenges such as unobserved returns, temporal dependencies, and distribution shifts. We further establish asymptotic lower bounds on coverage based on Wasserstein metrics and demonstrate the effectiveness of our method through empirical studies on both synthetic data and the Mountain Car environment.

\item {\bf Limitations}
    \item[] Question: Does the paper discuss the limitations of the work performed by the authors?
    \item[] Answer: \answerYes{}.
    \item[] Justification: The conclusion section outlines the limitations of the proposed method and proposes potential directions for future investigation.

\item {\bf Theory assumptions and proofs}
    \item[] Question: For each theoretical result, does the paper provide the full set of assumptions and a complete (and correct) proof?
    \item[] Answer: \answerYes{}.
    \item[] Justification: All the theorems, formulas, and proofs in the paper should be numbered and cross-referenced. All assumptions should be clearly stated or referenced in the statement of any theorems. Theorems and Lemmas that the proof relies upon should be properly referenced.

    \item {\bf Experimental result reproducibility}
    \item[] Question: Does the paper fully disclose all the information needed to reproduce the main experimental results of the paper to the extent that it affects the main claims and/or conclusions of the paper (regardless of whether the code and data are provided or not)?
    \item[] Answer: \answerYes{}.
    \item[] Justification: Section 4.1 specifies datasets, model sizes, hyper-parameters, and links (in Appendix) to an open GitHub repository, enabling faithful replication of the core experiments.

\item {\bf Open access to data and code}
    \item[] Question: Does the paper provide open access to the data and code, with sufficient instructions to faithfully reproduce the main experimental results, as described in supplemental material?
    \item[] Answer: \answerYes{}.
    \item[] Justification: : Code is publicly released on GitHub, and all referenced datasets are publicly available, with citations provided for each.

\item {\bf Experimental setting/details}
    \item[] Question: Does the paper specify all the training and test details (e.g., data splits, hyperparameters, how they were chosen, type of optimizer, etc.) necessary to understand the results?
    \item[] Answer: \answerYes{}.
    \item[] Justification: All the experimental settings are specified at the beginning and section 4, and details such as training and test sample sizes, DRL training details appear in implementation details of Section 4, offering sufficient context.

\item {\bf Experiment statistical significance}
    \item[] Question: Does the paper report error bars suitably and correctly defined or other appropriate information about the statistical significance of the experiments?
    \item[] Answer: \answerYes{}.
    \item[] Justification: : Results are reported as boxplots and medians.

\item {\bf Experiments compute resources}
    \item[] Question: For each experiment, does the paper provide sufficient information on the computer resources (type of compute workers, memory, time of execution) needed to reproduce the experiments?
    \item[] Answer: \answerYes{}.
    \item[] Justification: The paper provide the amount of compute required for each of the individual experimental runs as well as estimate the total compute. 
    
\item {\bf Code of ethics}
    \item[] Question: Does the research conducted in the paper conform, in every respect, with the NeurIPS Code of Ethics \url{https://neurips.cc/public/EthicsGuidelines}?
    \item[] Answer: \answerYes{}.
    \item[] Justification: The study uses only publicly available data, releases code responsibly, involves no human subjects, and follows standard ethical practice.

\item {\bf Broader impacts}
    \item[] Question: Does the paper discuss both potential positive societal impacts and negative societal impacts of the work performed?
    \item[] Answer: \answerNo{}.
    \item[] Justification:  The manuscript does not contain a Broader-Impact discussion of how the proposed method might be beneficial or misused.

\item {\bf Safeguards}
    \item[] Question: Does the paper describe safeguards that have been put in place for responsible release of data or models that have a high risk for misuse (e.g., pretrained language models, image generators, or scraped datasets)?
    \item[] Answer:  \answerNo{}.
    \item[] Justification: : Although the model is publicly released, the paper does not outline usage restrictions, filters, or other safeguards against malicious exploitation.

\item {\bf Licenses for existing assets}
    \item[] Question: Are the creators or original owners of assets (e.g., code, data, models), used in the paper, properly credited and are the license and terms of use explicitly mentioned and properly respected?
    \item[] Answer: \answerYes{}.
    \item[] Justification: Prior models and datasets are cited.

\item {\bf New assets}
    \item[] Question: Are new assets introduced in the paper well documented and is the documentation provided alongside the assets?
    \item[] Answer: \answerNA{}.
    \item[] Justification: The work does not release a new dataset; the proposed policy evaluation framework is a not provided as a separate asset.

\item {\bf Crowdsourcing and research with human subjects}
    \item[] Question: For crowdsourcing experiments and research with human subjects, does the paper include the full text of instructions given to participants and screenshots, if applicable, as well as details about compensation (if any)? 
    \item[] Answer: \answerNA{}.
    \item[] Justification: No crowdsourcing or human-subject research is involved.

\item {\bf Institutional review board (IRB) approvals or equivalent for research with human subjects}
    \item[] Question: Does the paper describe potential risks incurred by study participants, whether such risks were disclosed to the subjects, and whether Institutional Review Board (IRB) approvals (or an equivalent approval/review based on the requirements of your country or institution) were obtained?
    \item[] Answer: \answerNA{}.
    \item[] Justification: The study does not involve human subjects and therefore requires no IRB review.

\item {\bf Declaration of LLM usage}
    \item[] Question: Does the paper describe the usage of LLMs if it is an important, original, or non-standard component of the core methods in this research? Note that if the LLM is used only for writing, editing, or formatting purposes and does not impact the core methodology, scientific rigorousness, or originality of the research, declaration is not required.
    %this research? 
    \item[] Answer: \answerNA{}
    \item[] Justification: This research does not involve LLMs as any important, original, or non-standard components.

\end{enumerate}

%--------------------------------------------------------
\newpage

\appendix

\section{Preliminaries}

We impose the following standard assumptions in RL. In our notation, $\mathcal{P}$ denotes a probability distribution.

\begin{assumption}[Markov Property]
\label{mdp_assumption}
The decision process satisfies the Markov property: the next state and reward depend only on the current state and action. Formally, for all $t$,
\[
\mathcal{P}(S_{t+1}, R_t \mid A_t, S_t, R_{t-1}, A_{t-1}, S_{t-1}, \ldots, S_0) = \mathcal{P}(S_{t+1}, R_t \mid S_t, A_t).
\]
\end{assumption}

\begin{assumption}[Time-homogeneity]
\label{time_homogeneity}
The distribution of the transition and reward remains stationary over time. Specifically, for all $t$, the joint distribution of the next state and reward given the current state and action satisfies
\[
\mathcal{P}(S_{t+1}, R_t \mid S_t, A_t) = \mathcal{P}(S_{t}, R_{t-1} \mid S_{t-1}, A_{t-1}).
\]
\end{assumption}

\begin{assumption}[Stationary Policy]
The policy is stationary and Markovian: the action at each time step depends only on the current state and not on the full history. Formally, for all $t$,
\[
\pi_t(A_t \mid S_t, R_{t-1}, A_{t-1}, S_{t-1}, \ldots,S_0) = \pi(A_t \mid S_t).
\]
\end{assumption}

Before proceeding with theoretical analysis, we introduce the  distributional Bellman operator  and several related results. We use $\eta^\pi(s)$ to denote the distribution of the return starting from the initial state $s$ following policy $\pi$, that is,
\bas
\eta^\pi(s) 
:= \mathcal{P}^\pi ( G \,|\, S_0=s)
:=  \mathcal{P}^\pi (\sum_{t=0}^\infty \gamma^t R_t \,|\, S_0=s ).
\eas
We define the \textbf{distributional Bellman operator} $\mathcal{T}^\pi$ as the following transformation:
\bas
(\mathcal{T}^\pi\eta^\pi)(s) = \mathcal{P}^\pi( R+\gamma G^{\pi}(S') \,|\,  s )
\eas
where the transition $(s, R,S')$ is generated by sampling an action from $\pi$, observing the reward $R$,  and transitioning to the next state $S^\prime$, and $G^{\pi}(S^\prime)\sim \eta^{\pi}(S^\prime)$. 

Under the time-homogeneity assumption, $\eta^\pi$ satisfies the fixed-point condition:
\bas
\eta^\pi(s) = (\mathcal{T}^\pi\eta^\pi)(s), \quad \forall s\in\mathcal{S}.
\eas

A key property of the distributional Bellman operator $\mathcal{T}^\pi$ is that it is a $\gamma$-contraction w.r.t. the Wasserstein distance, stated in Proposition \ref{prop:contraction-prop}.  
The $p$-Wasserstein distance between two measures $\mu$ and $\nu$ on the real space $\mathbb{R}$ is defined as
\bas
W_p(\mu, \nu) 
:=  \inf_{\kappa \in \Gamma(\mu,\nu)} \left( \int_{\mathbb{R}\times \mathbb{R}} |x - y|^p \kappa(dx, dy)\right)^{1/p}
\eas
where $\Gamma(\mu, \nu)$ is the set of all couplings with marginals $\mu$ and $\nu$. We begin by presenting some fundamental results on the Wasserstein distance.

\begin{proposition}[Duality Formula for 1-Wasserstein Distance \citep{villani2008optimal}]
	\label{duality_wasserstein}
	For any measures $\mu$ and $\nu$,
	\bas
	W_1(\mu,\nu) = 
	\sup_{\psi: \Vert\psi\Vert_{\rm Lip}\le 1}
	\left\{ \int \psi\,d\mu - \int \psi\,d\nu \right\},
	\eas
	where  ``$ \Vert\psi\Vert_{\rm Lip}\le 1$''  means that  $\psi$ 
    is a 1-Lipschitz   function.
\end{proposition}

\begin{proposition}\label{prop:contraction-1-W}
	Suppose $\Vert f\Vert_{\rm Lip}\le 1$ and $b_f$ is an operator on measure such that $b_f(\mu)(A) = \mu(f^{-1}(A))$ for any measure $\mu$ and Borel set A. Then $b_f$ is a contraction under 1-Wasserstein distance, i.e., $W_1(b_f(\mu), b_f(\nu)) \leq W_1(\mu, \nu)$ for all measures  $ \mu, \nu$.
\end{proposition}
\begin{proof}
   For any 1-Lipschitz function $\psi$, the composition $\psi \circ f$ is also 1-Lipschitz, since the composition of 1-Lipschitz functions preserves the Lipschitz constant. By Proposition \ref{duality_wasserstein}, we have, for any measures $\mu$ and $\nu$,
   \bas
   W_1\left(b_f(\mu),b_f(\nu)\right) 
   &=& 	\sup_{\psi:\Vert\psi\Vert_{\rm Lip}\le 1}
   \left\{ \int \psi\circ f\,d\mu - \int \psi \circ f \,d\nu \right\} 
   \\
   &\le& 	\sup_{\psi:\Vert\tilde\psi\Vert_{\rm Lip}\le 1}
   \left\{ \int \tilde\psi\,d\mu - \int \tilde \psi\,d\nu \right\}
   \\
   &=& 	W_1(\mu,\nu),
   \eas
where the first equation follows from the change-of-variables formula for measures.
\end{proof}
Applying Proposition \ref{prop:contraction-1-W}, for any real random variables $X$ and $Y$ with laws $\mathcal{P}_{X}$ and $\mathcal{P}_{Y}$, since $f(x) = |x-a|$ for any $a\in\mathbb{R}$ is 1-Lipschitz continuous, we have $W_1(\mathcal{P}_{|X-a|},\mathcal{P}_{|Y-a|}) \le W_1(\mathcal{P}_X,\mathcal{P}_Y)$.

We now state the contraction property of the distributional Bellman operator $\mathcal{T}^\pi$. Let $\mathscr{P}$ be the set of all probability distributions over $\mathbb{R}$. Please note that the conditional return distribution given a  state ($s \in \mathcal{\mathcal{S}}$) is a  distribution that is indexed by the state $s$. That is, $\eta^\pi(\cdot) \in \mathscr{P}^{\mathcal{S}}$, and $\mathscr{P}^{\mathcal{S}}$ contains all possible conditional return distributions. We define the Wasserstein distance of two conditional distributions $\mu(\cdot), \nu(\cdot) \in \mathscr{P}^{\mathcal{S}}$ as $\bar W_p(\mu(\cdot),\nu(\cdot)) := \sup_{s \in \mathcal{S}} W_p(\mu(s),\nu(s))$.

% For any measure $\mu$ and $\nu$, we denote $\bar W_p(\mu,\nu) := \sup_{s \in \mathcal{S}} W_p(\mu,\nu)$. Let $\mathscr{P}$ be the set of all probability distributions over $\mathbb{R}$. {\red [Given a state $s$, for any measures $\eta(s)$ and $\eta^\prime(s)$, we define $\bar W_p(\eta(s),\eta^\prime(s)) := \sup_{s \in \mathcal{S}} W_p(\eta(s),\eta^\prime(s))$. Let $\mathscr{P}$ be the set of all probability distributions over $\mathbb{R}$].}

\begin{proposition}\label{prop:contraction-prop}[\cite{bellemare2023distributional}, Proposition 4.15]
	The distributional Bellman operator is a $\gamma$-contraction on $\mathscr{P}^\mathcal{S}$ w.r.t. the supreme $p$-Wasserstein metric for $p \in [1, \infty)$. That is, for any $\eta, \eta' \in \mathscr{P}^\mathcal{S}$ , we have $\bar W_p (\mathcal{T}^\pi\eta, \mathcal{T}^\pi\eta') \le \gamma \bar W_p (\eta, \eta')$. 
\end{proposition}

We denote the learned DRL model in the proposed prediction procedure by $\widehat\eta^\pi(s)$. It is clear that given $S_t$, the one-step pseudo-return $\widetilde{G}^{(1)}(S_t)=R_t+\gamma \widetilde{G}^{\pi}(S_{t+1})$ with $\widetilde{G}^{\pi}(S_{t+1}) \sim \widehat\eta^\pi(S_{t+1})$, follows the distribution $(\mathcal{T}^\pi\widehat\eta^\pi)(S_t)$. The following proposition shows that a similar conclusion also holds when the step width is $k$. That is, the $k$-step pseudo-return starting from $S_t$ follows $((\mathcal{T}^\pi)^k\widehat\eta^\pi)(S_t)$.

\begin{proposition}[\cite{bellemare2023distributional}, Lemma 4.33]
	\label{k_times_operator}
	Let $\eta \in \mathscr{P}^\mathcal{S}$, and let $G$ be an instantiation of $\eta$. For $s \in \mathcal{S}$, if $(S_t, A_t, R_t)_{t\ge0}$ is a random trajectory with initial state $S_0 = s$ and generated by following $\pi$, independent of $G$, then $\sum_{t=0}^{k-1} \gamma_tR_t + \gamma^kG(S_k)$ is an instantiation of  $((\mathcal{T}^\pi)^k\eta)(s)$. 
\end{proposition}

Proposition \ref{k_times_operator} allows us to investigate the $k$-step pseudo-return. As discussed in the main paper, we measure the coverage gap using the distributional distance between the estimated return distribution and the true return distribution. Unlike traditional approaches that rely on total variation distance, we adopt the Wasserstein distance, motivated by the insights in \cite{xu2025wasserstein}. A key intermediary that links the coverage error and the Wasserstein distance is the Kolmogorov distance, which is defined as follows.

%The next lemma is useful in proving the validity of the proposed PI. The Kolmogorov distance can serve as a mediator between the coverage error and the wasserstein distance.

\begin{definition}[Kolmogorov Distance]
	$F_\mu$ and $F_\nu$ are the CDFs of probability measures $\mu$ and $\nu$ on $\mathbb{R}$, respectively. Kolmogorov distance between $\mu$ and $\nu$ is given by 
    \bas
    K(\mu, \nu) = \sup_{x\in\mathbb{R}} |F_\mu(x) - F_\nu (x)|.
    \eas
\end{definition}
 
\begin{lemma}[\cite{ross2011fundamentals}]
	If a probability measure $\mu$ in space $\mathbb{R}$ has Lebesgue density bounded by $L$, then for any probability measure $\nu$, $K(\mu, \nu) \le \sqrt{2LW_1(\mu, \nu)}$.
\end{lemma}

\section{Proof of Theorem 1}

We now present the proof of the main theorem for the proposed PIs in the on-policy evaluation setting.

\begin{proof}[Proof of Theorem 1]
		Since $\widehat{C}_{N,\alpha}^{\rm on}(S_{\rm test})$ combines $B$ intervals following \cite{zhang2023conformal,solari2022multi}, it suffices to prove the validity of each single CP interval. 
	With some abuse of notation, we denote the single CP interval at target coverage level $1-\alpha$ as $\widehat C_{N,\alpha}^{\rm on}(S_{\rm test})$.
	
	{We first consider the case where data splitting is performed in a trajectory-wise manner, and let $n$ denote the number of trajectories in the calibration set ${\mathcal{D}}_{\rm cal}$. We index the trajectories in the calibration dataset ${\mathcal{D}}_{\rm cal}$ as $\{1,2,\ldots,n\}$.} Please note that, with a slight abuse of notation, $n$ here denotes the number of trajectories, which differs from its definition in the main paper. In the main paper, $n$ refers to the cardinality of the calibration set $\mathcal{D}_{\mathrm{cal}}$, where data are stored as tuples rather than trajectories.
    
    Note that the step-width in constructing the pseudo-return is $k$. For a state variable $S_{it}$ in the data, the corresponding pseudo-return is constructed as 
	\bas
	\widetilde G^{(k)}(S_{it}) 
	:= \sum_{h=0}^{k-1} \gamma^h R_{i,t+h} 
	+ \gamma^k \widetilde{G}^{\pi}(S_{i,t+k}),
	 \quad \widetilde{G}^{\pi}(S_{i,t+k}) \sim \widehat\eta^\pi(S_{i,t+k}).
	\eas
Hereafter, for notational simplicity, we denote $\widetilde{G}_{it}^{(k)}:=\widetilde G^{(k)}(S_{it}) $.
By Proposition \ref{k_times_operator}, 
	$$\widetilde G_{it}^{(k)} 
	\sim  ((\mathcal{T}^\pi)^k \widehat\eta^\pi)(S_{it}).$$ 
    
    %We denote $\widetilde{\eta}^{(k)}(S_{it}):=((\mathcal{T}^\pi)^k \widehat\eta^\pi)(S_{it})$ for simplicity of notation.
	
	Given all the data $\mathcal{D}$, the calibration set $\widetilde{\mathcal{D}}_{\rm cal}$, using experience replay and weighted subsampling, is a set of samples drawn from the distribution:
	\bas
	\widehat F_n(s,g) := 
	\sum_{t=0}^{T-k} \sum_{i=1}^{n} 
	\frac{\widehat w_{\rm on}(S_{it})}{	\sum_{t=0}^{T-k} 
		\sum_{j=1}^{n} \widehat w_{\rm on}(S_{jt}) }
		 I\{S_{it} \le s, \widetilde G_{it}^{(k)} \le g \} .
	\eas
  \paragraph{Main idea.}
The proof proceeds by successively isolating the effects of the two estimation errors: the approximation of $\eta^{\pi}(s)$ and the estimation of the weighting function. For notational simplicity, we abbreviate the return $G^\pi(S_{\rm test})$ on the test data as $G_{\rm test}$.

We begin by noting that the true test point is drawn from
\[
  (S_{\rm test}, G_{\rm test}) \sim \mathcal{P}_{S_0} \times ((\mathcal{T}^\pi)^k \eta^\pi)(S_{0}),
\]
where $S_0$ is the random initial state with marginal distribution $\mathcal{P}_{S_0}$. To quantify the error induced by approximating $\eta^{\pi}(s)$, we introduce an intermediate test point
\[
  (S_{\rm test}, \widetilde{G}_{\rm test}) \sim \mathcal{P}_{S_0} \times ((\mathcal{T}^\pi)^k \widehat{\eta}^\pi)(S_{0}),
\]
which shares the same state distribution as the true test point but replaces $\eta^{\pi}$ with its estimator $\widehat{\eta}^{\pi}$ (see (2) of this proof for details).

Next, to analyze the additional error due to weight estimation, we define another artificial test point
\[
  (\widehat{S}_{\rm test}, \widehat{G}_{\rm test}) \sim \widehat{F}_n(s,g),
\]
which differs from $(S_{\rm test}, \widetilde{G}_{\rm test})$ only in the state distribution (see (3) of this proof for details). 

Finally, conditional on $\mathcal{D}$, $(\widehat{S}_{\rm test}, \widehat{G}_{\rm test})$ is exchangeable with $\widetilde{\mathcal{D}}_{\rm cal}$. Hence, the standard conformal prediction argument applies, establishing the conditional coverage property in Eq.~\eqref{eqn:conditional-coverage}.

Given these new test points, we can bound the coverage probability of $G_{\rm test} := G^{\pi}(S_{\rm test})$ as
\begin{align*}
     & \Pr \left( G_{\rm test} \in \widehat{C}_{N,\alpha}^{\rm on}(S_{\rm test}) \right)\ge \Pr \left( \widehat G_{\rm test} \in \widehat{C}_{N,\alpha}^{\rm on}(\widehat S_{\rm test}) \right)  
      \\
     &\quad
     - \Bigg| 
     \Pr \left( \widehat G_{\rm test} \in \widehat{C}_{N,\alpha}^{\rm on}(\widehat S_{\rm test}) \right) 
     - \Pr \left( \widetilde G_{\rm test} \in \widehat{C}_{N,\alpha}^{\rm on}(S_{\rm test}) \right) 
     \Bigg|
     \\
     & \quad - \Bigg| 
     \Pr \left( \widetilde G_{\rm test} \in \widehat{C}_{N,\alpha}^{\rm on}( S_{\rm test}) \right) 
     - \Pr \left( G_{\rm test} \in \widehat{C}_{N,\alpha}^{\rm on}(S_{\rm test}) \right) 
     \Bigg|
     \\
     &:= M_{1} - M_{2} - M_{3}.
\end{align*}
We now analyze $M_1$, $M_2$ and $M_3$ individually.
	% We consider two new test points: 
	% \bas
	% (\widehat S_{\rm test},\widehat{G}_{\rm test}) 
	% \sim  
	% \widehat F_n(s,g),
	% \eas
	% and
	% \bas
	% (S_{\rm test},\widetilde G_{\rm test}) \sim \mathcal{P}_{S_0} \times ((\mathcal{T}^\pi)^k \widehat\eta^\pi)(S_{0}),
	% \eas
 %     which are drawn independently and $S_{0}$ is the random variable of the initial state with marginal distribution $\mathcal{P}_{S_0}$. 
 %     For notational simplicity, we abbreviate the return $G^\pi(S_{\rm test})$ on the test data as $G_{\rm test}$. Then
 %     \bas
 %     & & \Pr \left( G_{\rm test} \in \widehat{C}_{N,\alpha}^{\rm on}(S_{\rm test}) \right)  
 %      \\
 %     &\ge&  \Pr \left( \widehat G_{\rm test} \in \widehat{C}_{N,\alpha}^{\rm on}(\widehat S_{\rm test}) \right) 
 %     - \left| 
 %     \Pr \left( \widehat G_{\rm test} \in \widehat{C}_{N,\alpha}^{\rm on}(\widehat S_{\rm test}) \right) 
 %     - \Pr \left( \widetilde G_{\rm test} \in \widehat{C}_{N,\alpha}^{\rm on}(S_{\rm test}) \right) 
 %     \right|
 %     \\
 %     & & - \left| 
 %     \Pr \left( \widetilde G_{\rm test} \in \widehat{C}_{N,\alpha}^{\rm on}( S_{\rm test}) \right) 
 %     - \Pr \left( G_{\rm test} \in \widehat{C}_{N,\alpha}^{\rm on}(S_{\rm test}) \right) 
 %     \right|
 %     \\
 %     &:=& M_{1} - M_{2} - M_{3}.
 %     \eas
 %      We now analyze $M_1,M_2,M_3$ individually.
      
     \noindent\textbf{(1)} Given $\mathcal{D}$, $(\widehat S_{\rm test},\widehat{G}_{\rm test})$ is exchangeable with $\widetilde{\mathcal{D}}_{cal}$.
     Then, existing conclusions about coverage rate in SCP \citep{lei2018distribution} gives
     \ba\label{eqn:conditional-coverage}
     \Pr \left( \widehat G_{\rm test} \in \widehat{C}_{N,\alpha}^{\rm on}(\widehat S_{\rm test}) \,|\, \mathcal{D}  \right)
     \ge 1-\alpha.
     \ea
     Taking expectation for the above inequality  gives
     \ba
     M_1 \ge 1-\alpha.
     \label{Thm1_M1}
     \ea 
     
      \noindent\textbf{(2)} Recall that $\widehat{C}_{N,\alpha}^{\rm on}(S_{\rm test}) = \widehat v^{\pi}(S_{\rm test})\pm \widehat q_{1-\alpha}$. By Propositions 2-4 and Lemma 1,
     \bas
     & &\left|\Pr\left(\widetilde G_{\rm test} \in \widehat{C}_{N,\alpha}(S_{\rm test}) \,|\, \mathcal{D}_{\rm tr}, \widetilde{\mathcal{D}}_{\rm cal},S_{\rm test} \right)  -\Pr\left(G_{\rm test} \in \widehat{C}_{N,\alpha}(S_{\rm test}) \,|\, \mathcal{D}_{\rm tr}, \widetilde{\mathcal{D}}_{\rm cal},S_{\rm test}\right) \right| \\
     &=& \left| F_{|\widetilde G_{\rm test}-\widehat v^\pi(S_{\rm test})|}(\widehat q_{1-\alpha})
     - F_{| G_{\rm test}-\widehat v^\pi(S_{\rm test})|}(\widehat q_{1-\alpha}) \right| 
     \\
     &\le & K\left( \mathcal{P}_{|\widetilde G_{\rm test}-\widehat v^\pi(S_{\rm test})|},
     \mathcal{P}_{| G_{\rm test}-\widehat v^\pi(S_{\rm test})|} \right)\quad\text{by Definition 1,}\\
     & \le & \sqrt{ 2L 
     	W_1\left( \mathcal{P}_{|\widetilde G_{\rm test}-\widehat v^\pi(S_{\rm test})|},
     	\mathcal{P}_{| G_{\rm test}-\widehat v^\pi(S_{\rm test})|} \right) }\quad\text{by Lemma 1,}
     \\
     &\le& \sqrt{ 2L  W_1\left( \mathcal{P}_{\widetilde G_{\rm test}}, \mathcal{P}_{G_{\rm test}} \right) }\quad\text{by Proposition 2,}
     \\
     &\le& \sqrt{ 2L \bar{W}_1\left( (\mathcal{T}^\pi)^k\widehat\eta^\pi,(\mathcal{T}^\pi)^k\eta^\pi \right) }\quad\text{by Proposition 4,}
     \\
     &\le& \sqrt{ 2L \gamma^k \bar W_1(\widehat\eta^\pi, \eta^\pi)}\quad\text{by Proposition 3.}
     \eas
Since \( f(x) = \sqrt{x} \) is a concave function, taking expectations on both sides of the inequality and applying Jensen's inequality yields:
\begin{align}
    M_3 
    &= \left| \mathbb{E} \left[ \Pr\left(\widetilde{G}_{\text{test}} \in \widehat{C}_{N,\alpha}(S_{\text{test}}) \,\middle|\, \mathcal{D}_{\text{tr}}, \widetilde{\mathcal{D}}_{\text{cal}}, S_{\text{test}} \right) 
    - \Pr\left(G_{\text{test}} \in \widehat{C}_{N,\alpha}(S_{\text{test}}) \,\middle|\, \mathcal{D}_{\text{tr}}, \widetilde{\mathcal{D}}_{\text{cal}}, S_{\text{test}} \right) \right] \right| \notag \\
    &\le \mathbb{E} \left| \Pr\left(\widetilde{G}_{\text{test}} \in \widehat{C}_{N,\alpha}(S_{\text{test}}) \,\middle|\, \mathcal{D}_{\text{tr}}, \widetilde{\mathcal{D}}_{\text{cal}}, S_{\text{test}} \right) 
    - \Pr\left(G_{\text{test}} \in \widehat{C}_{N,\alpha}(S_{\text{test}}) \,\middle|\, \mathcal{D}_{\text{tr}}, \widetilde{\mathcal{D}}_{\text{cal}}, S_{\text{test}} \right) \right| \notag \\
    &\le \mathbb{E} \left[ \sqrt{2L \gamma^k \, \bar{W}_1\left(\widehat{\eta}^\pi, \eta^\pi \right)} \right]\le \sqrt{2L \gamma^k \, \mathbb{E} \left[ \bar{W}_1\left(\widehat{\eta}^\pi, \eta^\pi \right) \right]}\quad\text{by Jensen's inequality.}
    \label{eq:Thm1_M3}
    % \notag \\
    % &\le \sqrt{2L \gamma^k \, \mathbb{E} \left[ \bar{W}_1\left(\widehat{\eta}^\pi, \eta^\pi \right) \right]}\quad\text{by Jensen's inequality.}
    % \label{eq:Thm1_M3}
\end{align}

%{\blue Since $f(x)=\sqrt{x}$ is a concave function, taking expectation for the above inequality and then applying Jensen's inequality gives
%    \ba
%     M_3 
%     &=& \left| \e \left[\Pr\left(\widetilde G_{\rm test} \in \widehat{C}_{N,\alpha}(S_{\rm test}) \,|\, \mathcal{D}_{tr}, \widetilde{\mathcal{D}}_{cal},S_{\rm test} \right)  -\Pr\left(G_{\rm test} \in \widehat{C}_{N,\alpha}(S_{\rm test}) \,|\, \mathcal{D}_{tr}, \widetilde{\mathcal{D}}_{cal},S_{\rm test}\right) \right] \right|
 %    \notag \\
%      &\le& \e \left|\Pr\left(\widetilde G_{\rm test} \in \widehat{C}_{N,\alpha}(S_{\rm test}) \,|\, \mathcal{D}_{tr}, \widetilde{\mathcal{D}}_{cal},S_{\rm test} \right)  -\Pr\left(G_{\rm test} \in \widehat{C}_{N,\alpha}(S_{\rm test}) \,|\, \mathcal{D}_{tr}, \widetilde{\mathcal{D}}_{cal},S_{\rm test}\right) \right|
%     \notag \\
%     &\le& \e\left[\sqrt{ 2L \gamma^k  \bar W_1(\widehat\eta^\pi,\eta^\pi)}\right]
%     \notag\\
%     &\le& \sqrt{ 2L \gamma^k  \e[\bar W_1(\widehat\eta^\pi,\eta^\pi)]}.
%     \label{Thm1_M3}
%   \ea }
    
    \noindent\textbf{(3)}  
    Let \(\mathcal{P}_t(s, g)\) denote the distribution of \((S_t, \widetilde{G}_t^{(k)})\) conditioned on \(\mathcal{D}_{{\rm tr}}\). While the marginal distribution of \(S_t\) may vary across time steps, the conditional distribution of \(\widetilde{G}_t^{(k)} \mid S_t\) remains time-homogeneous. Now we analyze $M_2$ and first define $M_{2}(\mathcal{D}, \widetilde{\mathcal{D}}_{\rm cal})$ as follows.    
    %{\blue Denote the probability distribution of $(S_{t},\widetilde G_{t}^{(k)})$ given $\mathcal{D}_{tr}$ as $\mathcal{P}_t(s,g)$. The marginal distribution of $S_t$ is different across time points, while the conditional distribution of $G_{t}^{(k)}|S_{t}$ is homogeneous across $t$.}
  \bas
    M_{2}(\mathcal{D}, \widetilde{\mathcal{D}}_{\rm cal})  &:=  & \left| 
     \Pr \left( \widehat G_{\rm test} \in \widehat{C}_{N,\alpha}^{\rm on}(\widehat S_{\rm test}) \mid \mathcal{D}, \widetilde{\mathcal{D}}_{\rm cal} \right) 
    - \Pr \left( \widetilde G_{\rm test} \in \widehat{C}_{N,\alpha}^{\rm on}(S_{\rm test})
     \mid \mathcal{D}, \widetilde{\mathcal{D}}_{\rm cal} \right) 
    \right| 
    \\
    &=& 
   \Big|
    \sum_{t=0}^{T-k} \sum_{i=1}^{n} 
    \frac{\widehat w_{\rm on}(S_{it})}{\sum_{t=0}^{T-k} \sum_{j=1}^{n} 
    	\widehat w_{\rm on}(S_{jt})}
    I\left\{ |\widetilde{G}^{(k)}_{it} - \widehat v^\pi(S_{it}) | \le \widehat q_{1-\alpha} \right\}\\ 
   &&\quad\quad - \Pr\left(|\widetilde G_{\rm test} - \widehat v^{\pi}(S_{\rm test})| \le \widehat q_{1-\alpha}  \mid \mathcal{D},\widetilde{\mathcal{D}}_{\rm cal} \right) 
   \Big|\le M_{21}+M_{22},
    \eas  
where
	\bas
	&&M_{21} := \sup_{x \in \mathbb{R}} \left|
	\sum_{t=0}^{T-k} \sum_{i=1}^{n} 
	\frac{\widehat w_{\rm on}(S_{it})}{\sum_{t=0}^{T-k} \sum_{j=1}^{n} \widehat w_{\rm on}(S_{jt})}
	I\left\{ |\widetilde{G}^{(k)}_{it} - \widehat v^\pi(S_{it}) | \le x \right\} - B(x\mid \mathcal{D},\widetilde{\mathcal{D}}_{\rm cal})
	\right|,
	\\
	&&M_{22} := \left|
	B(\widehat q_{1-\alpha}\mid \mathcal{D},\widetilde{\mathcal{D}}_{\rm cal}) - \Pr\left(|\widetilde G_{\rm test} - \widehat v^{\pi}(S_{\rm test})| \le \widehat q_{1-\alpha}  \mid \mathcal{D},\widetilde{\mathcal{D}}_{\rm cal} \right)
    \right|,\quad\text{where}\\
     &&B(x\mid \mathcal{D},\widetilde{\mathcal{D}}_{\rm cal}):= \frac{1}{T-k+1} \sum_{t=0}^{T-k} \int \widehat w_{\rm on}(s) I\{|g-\widehat v^\pi(s)| \le x \}\, d\mathcal{P}_t(s,g). 
	\eas
\noindent\textbf{(3.1)}  To analyze $M_{21}$, we first define the normalization constant for weights as $$W_{n}=\frac{1}{n(T-k+1)}\sum_{t=0}^{T-k}\sum_{i=1}^n \widehat w_{\rm on}(S_{it}).$$
Thus the first term in $M_{21}$ becomes 
\bas
\frac{1}{W_{n}}\frac{1}{n(T-k+1)}\sum_{t=0}^{T-k}\sum_{i=1}^{n}\widehat w_{\rm on}(S_{it})I\left\{ |\widetilde{G}^{(k)}_{it} - \widehat v^\pi(S_{it}) | \le x \right\}:=\frac{1}{W_{n}}B_{\rm emp}(x\mid \mathcal{D},\widetilde{\mathcal{D}}_{\rm cal}),
\eas
where $B_{\rm emp}(x\mid \mathcal{D},\widetilde{\mathcal{D}}_{\rm cal})$ is an empirical version of $B(x\mid \mathcal{D},\widetilde{\mathcal{D}}_{\rm cal})$. By a simple algebraic calculation, we have
\begin{align*}
M_{21} 
&\le \frac{1}{W_{n}} \sup_{x \in \mathbb{R}} 
\left| B_{\rm emp}(x \mid \mathcal{D}, \widetilde{\mathcal{D}}_{\rm cal}) 
- B(x \mid \mathcal{D}, \widetilde{\mathcal{D}}_{\rm cal}) \right|  + \left( \frac{1}{W_{n}} - 1 \right) 
\sup_{x \in \mathbb{R}} \left| B(x \mid \mathcal{D}, \widetilde{\mathcal{D}}_{\rm cal}) \right|.
\end{align*}
    Since $\frac{1}{T-k+1}\sum_{t=0}^{T-k}\e [\widehat w_{\rm on}(S_{t})\mid \mathcal{D}_{\rm tr}] = 1$, by law of large numbers,  % central limit theorem
    \ba
    \lim_{n\to\infty}W_{n}
    = \frac{1}{T-k+1}\sum_{t=0}^{T-k} \e [\widehat w_{\rm on}(S_{t})\mid \mathcal{D}_{\rm tr}] = 1.
    \label{equ:weights-limits-1}
    \ea  
Hence, for sufficiently large $n$, $W_{n}\ge 1/2$ and
\begin{align*}
M_{21} 
&\le 2\underbrace{ \sup_{x \in \mathbb{R}} 
\left| B_{\rm emp}(x \mid \mathcal{D}, \widetilde{\mathcal{D}}_{\rm cal}) 
- B(x \mid \mathcal{D}, \widetilde{\mathcal{D}}_{\rm cal}) \right|}_{E} + 
\underbrace{ { \left| \frac{1}{W_{n}} - 1 \right| } \sup_{x \in \mathbb{R}} \left| B(x \mid \mathcal{D}, \widetilde{\mathcal{D}}_{\rm cal}) \right|}_{F}.
\end{align*}
\textbf{(3.1.1)} For $E$, since $\e[\widehat w_{\rm on}(S_{it}) \mid \mathcal{D}_{\rm tr}] < \infty$ for $0 \le t \le T-k$, the function class $\{ \widehat w_{\rm on} (s)I\{|g - \widehat v^\pi(s)|\le x\}: x \in \mathbb{R}\}$ is $\{  \mathcal{P}_t(s,g) : 0 \le t \le T-k\}$-Glivenko-Cantelli. Therefore, for all $0 \le t \le T-k$,
    \bas
     \lim_{n\to\infty}\sup_{x \in \mathbb{R}}\bigg|
     \frac{1}{n} \sum_{i=1}^n \widehat w_{\rm on}(S_{it}) I\left\{ |\widetilde G_{it}^{(k)} - \widehat v^\pi(S_{it}) | \le x \right\}
     - \int \widehat w_{\rm on}(s) I\left\{ |g - \widehat v^\pi(s) | \le x \right\} \, d\mathcal{P}_t(s,g)
     \bigg|
     = 0.
    \eas
  Averaging over $t$ gives
    \ba
    \lim_{n\to\infty} \sup_{x \in \mathbb{R}}\left|
      B_{\rm emp}(x\mid \mathcal{D},\widetilde{\mathcal{D}}_{\rm cal}) -  B(x\mid \mathcal{D},\widetilde{\mathcal{D}}_{\rm cal})
     \right|
     = 0.
     \label{equ:average-over-time-1}
    \ea 
\textbf{(3.1.2)} For $F$, we have { $\lim_{n\to\infty} \left( {1}/{W_{n}} - 1 \right) = 0$ by  Eq.\eqref{equ:weights-limits-1} and }
\bas
\sup_{x \in \mathbb{R}} \left| B(x \mid \mathcal{D}, \widetilde{\mathcal{D}}_{\rm cal}) \right|\le  \frac{1}{T-k+1} \sum_{t=0}^{T-k} \int \widehat w_{\rm on}(s)\, d\mathcal{P}_t(s,g)=\frac{1}{T - k + 1} \sum_{t=0}^{T - k} 
\mathbb{E}\left[ \widehat{w}_{\rm on}(S_t) \mid \mathcal{D}_{\rm tr} \right]=1, 
\eas
 by Eq.\eqref{equ:weights-limits-1}.
Then combining \eqref{equ:average-over-time-1}, we conclude that
\begin{equation}
\lim_{n \to \infty} M_{21} = 0.
\label{Thm1_M21}
\end{equation}

 \noindent\textbf{(3.2) Bound on $M_{22}$.} Recall that
\begin{align*}
&M_{22} := \Bigg|
B(\widehat q_{1-\alpha}\mid \mathcal{D},\widetilde{\mathcal{D}}_{\rm cal}) 
- \Pr\Big(|\widetilde G_{\rm test} - \widehat v^{\pi}(S_{\rm test})| \le \widehat q_{1-\alpha} \mid \mathcal{D},\widetilde{\mathcal{D}}_{\rm cal} \Big)
\Bigg|,\quad \text{where} \\
&B(x\mid \mathcal{D},\widetilde{\mathcal{D}}_{\rm cal}) := \frac{1}{T-k+1} \sum_{t=0}^{T-k} \int \widehat w_{\rm on}(s) I\{|g-\widehat v^\pi(s)| \le x \}\, d\mathcal{P}_t(s,g),
\end{align*}
where $\mathcal{P}_t(s,g)$ denotes the conditional distribution of $(S_t, \widetilde{G}_t^{(k)})$ given the training data $\mathcal{D}_{\rm tr}$.  

Define a new probability measure
\[
\frac{1}{T-k+1} \sum_{t=0}^{T-k} \widehat{w}_{\rm on}(s) \, d\mathcal{P}_t(s,g),
\]
and let $(\widetilde{S},\widetilde{G})$ be drawn from this measure. Then $M_{22}$ can be equivalently written as
\begin{align*}
M_{22} = \Bigg| 
\Pr\Big(|\widetilde G - \widehat v^{\pi}(\widetilde{S})| \le \widehat q_{1-\alpha} \mid \mathcal{D},\widetilde{\mathcal{D}}_{\rm cal} \Big)
- \Pr\Big(|\widetilde G_{\rm test} - \widehat v^{\pi}(S_{\rm test})| \le \widehat q_{1-\alpha} \mid \mathcal{D},\widetilde{\mathcal{D}}_{\rm cal} \Big)
\Bigg|.
\end{align*}

Since the conditional distributions $\widetilde{G}\mid \widetilde{S}$ and $\widetilde{G}_{\rm test}\mid S_{\rm test}$ are identical, by Eq.~(A.9) in \cite{lei2021conformal}, we have
\[
M_{22} \le d_{TV}(\mathcal{P}_{\widetilde{S}}, \mathcal{P}_{S_{\rm test}}),
\]
where $d_{TV}$ denotes the total variation distance.  

Denote the marginal distribution of $\mathcal{P}_t(s,g)$ as $\mathcal{P}_t(s)$ and define the calibration marginal $\mathcal{P}_{\rm cal}(s) = \frac{1}{T-k+1} \sum_{t=0}^{T-k} \mathcal{P}_t(s)$. Then $\widetilde{S} \sim \widehat{w}_{\rm on}(s) \mathcal{P}_{\rm cal}(s)$ and $S_{\rm test} \sim w_{\rm on}(s) \mathcal{P}_{\rm cal}(s)$. It follows that
\begin{align}
M_{22} &\le \frac12 \int \big|\widehat w_{\rm on}(s) - w_{\rm on}(s) \big| \, d\mathcal{P}_{\rm cal}(s) \notag \\
&= \frac{1}{2(T-k+1)} \sum_{t=0}^{T-k} \mathbb{E}\Big[ \big| \widehat w_{\rm on}(S_t) - w_{\rm on}(S_t) \big| \mid \mathcal{D}_{\rm tr} \Big],
\label{Thm1_M22}
\end{align}
where the last equality follows directly from the definition of $\mathcal{P}_{\rm cal}(s)$.

The desired result in Theorem 1 follows from  \eqref{Thm1_M1} -  \eqref{Thm1_M22}.

    \paragraph{Extension.} We now extend the above arguments to the setting where data splitting is performed at the tuple level---that is, on tuples of the form \((S_{it}, A_{it}, R_{it}, \ldots, S_{i,t+k})\), for \(1 \le i \le N\) and \(0 \le t \le T - k\). Let $n$ denote the number of tuples in $\mathcal{D}_{\rm cal}$, and let $n_t$ be the number of $t$-stage tuples included. Then it holds that $\sum_{t=0}^{T-k}n_t = n$. We index the data points of the $t$-th stage separately as $\{1,2,\ldots,n_t\}$ for notational simplicity. Given all data $\mathcal{D}$, $\widetilde{\mathcal{D}}_{\rm cal}$ is a set of sample drawn from
    \bas
      \widehat F_n^{*}(s,g) := 
	\sum_{t=0}^{T-k} \sum_{i=1}^{n_t} 
	\frac{\widehat w_{\rm on}(S_{it})}{	\sum_{t=0}^{T-k} 
		\sum_{j=1}^{n_t} \widehat w_{\rm on}(S_{jt}) }
		 I\{S_{it} \le s, \widetilde G_{it}^{(k)} \le g \} .
    \eas
 Similarly we consider three new points
 \bas
	(\widehat S_{\rm test}^*,\widehat{G}_{\rm test}^*) 
	\sim  
	\widehat F_n^{*}(s,g),\quad (S_{\rm test},\widetilde G_{\rm test})\sim \mathcal{P}_{S_0}\times ((\mathcal{T}^\pi)^k\widehat{\eta}^\pi)(S_0),\quad (S_{\rm test},G_{\rm test})\sim \mathcal{P}_{S_0}\times \eta^{\pi}(S_0).
	\eas

   Then the coverage probability satisfies:
     \bas
     & & \Pr \left( G_{\rm test} \in \widehat{C}_{N,\alpha}^{\rm on}(S_{\rm test}) \right)\ge  \Pr \left( \widehat G_{\rm test}^* \in \widehat{C}_{N,\alpha}^{\rm on}(\widehat S_{\rm test}^*) \right)  \\
     & &\quad  - \left| 
     \Pr \left( \widehat G_{\rm test}^* \in \widehat{C}_{N,\alpha}^{\rm on}(\widehat S_{\rm test}^*) \right) 
     - \Pr \left( \widetilde G_{\rm test} \in \widehat{C}_{N,\alpha}^{\rm on}(S_{\rm test}) \right) 
     \right|
     \\
     & &\quad - \left| 
     \Pr \left( \widetilde G_{\rm test} \in \widehat{C}_{N,\alpha}^{\rm on}( S_{\rm test}) \right) 
     - \Pr \left( G_{\rm test} \in \widehat{C}_{N,\alpha}^{\rm on}(S_{\rm test}) \right) 
     \right|
     \\
     &&:= M_{1}^{*} - M_{2}^{*} - M_{3}.
     \eas
 The analysis of $M_1^*$ mirrors that of $M_1$, and the treatment of $M_3$ remains unchanged from the previous case. We now focus on the detailed analysis of $M_2^*$. Similarly we define  $M_2^*(\mathcal{D},\widetilde{\mathcal{D}}_{\rm cal})$ as follows:
  \begin{align*}
       M_2^*(\mathcal{D},\widetilde{\mathcal{D}}_{\rm cal}) 
     &:=\left| 
     \Pr \left( \widehat G_{\rm test}^* \in \widehat{C}_{N,\alpha}^{\rm on}(\widehat S_{\rm test}^*) \mid \mathcal{D},\widetilde{\mathcal{D}}_{\rm cal} \right) 
     - \Pr \left( \widetilde G_{\rm test} \in \widehat{C}_{N,\alpha}^{\rm on}(S_{\rm test})  \mid \mathcal{D},\widetilde{\mathcal{D}}_{\rm cal}  \right) 
     \right|  \\
     & =  \left|
	\sum_{t=0}^{T-k} \sum_{i=1}^{n_t} 
	\frac{\widehat w_{\rm on}(S_{it})}{\sum_{t=0}^{T-k} \sum_{j=1}^{n_t} \widehat w_{\rm on}(S_{jt})}
	I\left\{ |\widetilde{G}^{(k)}_{it} - \widehat v^\pi(S_{it}) | \le \hat q_{1-\alpha} \right\} \right. \\
    &  \quad\quad\quad  \left. - \Pr \left( \widetilde G_{\rm test} \in \widehat{C}_{N,\alpha}^{\rm on}(S_{\rm test})  \mid \mathcal{D},\widetilde{\mathcal{D}}_{\rm cal}  \right)  \right| 
     \le   M_{21}^*+M_{22} 
\end{align*} 
where
\begin{align*}
       M_{21}^{*} 
     := \sup_{x \in \mathbb{R}} \left|
	\sum_{t=0}^{T-k} \sum_{i=1}^{n_t} 
	\frac{\widehat w_{\rm on}(S_{it})}{\sum_{t=0}^{T-k} \sum_{j=1}^{n_t} \widehat w_{\rm on}(S_{jt})}
	I\left\{ |\widetilde{G}^{(k)}_{it} - \widehat v^\pi(S_{it}) | \le x \right\}-B(x\mid \mathcal{D},\widetilde{\mathcal{D}}_{\rm cal}) \right|.
     \end{align*}

Then, we introduce an intermediate value for each time point $t$:
\bas
B_{\rm emp}(x\mid t, \mathcal{D},\widetilde{\mathcal{D}}_{\rm cal}) := \frac{1}{n_t}\sum_{i=1}^{n_t}\widehat w_{\rm on}(S_{it})I\left\{ |\widetilde{G}^{(k)}_{it} - \widehat v^\pi(S_{it}) | \le x \right\}, 
\eas
which is an empirical version of $B(x\mid t,\mathcal{D},\widetilde{\mathcal{D}}_{\rm cal})$ defined similarly:
\bas
B(x\mid t, \mathcal{D},\widetilde{\mathcal{D}}_{\rm cal}):=  \int \widehat w_{\rm on}(s) I\{|g-\widehat v^\pi(s)| \le x \}\, d\mathcal{P}_t(s,g). 
\eas

Let \(n_t\) denote the number of tuples at time step \(t\), for \(0 \le t \le T - k\). The vector \((n_0, n_1, \ldots, n_{T-k})\) follows a multinomial distribution with total count \(n\) and uniform probabilities over the \(T - k + 1\) time steps:
\[
(n_0, n_1, \ldots, n_{T-k}) \sim \text{Multinomial}\left(n; \, \left\{\tfrac{1}{T - k + 1}, \ldots, \tfrac{1}{T - k + 1} \right\}\right).
\]

As $n\rightarrow\infty$, it follows that $n_t\rightarrow\infty$ for all $t$. Applying the same argument as in Equation~\eqref{equ:average-over-time-1}, we obtain
\ba
    \lim_{n\to\infty} \sup_{x \in \mathbb{R}}\left|
      \frac{1}{T-k+1}\left\{B_{\rm emp}(x\mid t, \mathcal{D},\widetilde{\mathcal{D}}_{\rm cal})-B(x\mid t, \mathcal{D},\widetilde{\mathcal{D}}_{\rm cal})\right\}
     \right|
     = 0.\label{equ:tupe-splitting1}
    \ea 

Define the new normalization constant for weights as
$$W_{n}^*=\frac{1}{n}\sum_{t=0}^{T-k}\sum_{i=1}^{n_t}\widehat{w}_{\rm on}(S_{it}).$$ Since $\lim_{n\to\infty} n_t/n = 1/(T-k+1)$, 
it follows from law of large numbers that 
\ba
    \lim_{n\to\infty}W_{n}^*
    = \lim_{n\to\infty}\sum_{t=0}^{T-k}\frac{n_t}{n}
    \cdot\frac{1}{n_t} \sum_{i=1}^{n_t} \widehat w_{\rm on}(S_{it})
    = \frac{1}{T-k+1}\sum_{t=0}^{T-k} \e [\widehat w_{\rm on}(S_{t})\mid \mathcal{D}_{\rm tr}] = 1.\label{equ:tupe-splitting2}
    \ea
    %So, for sufficiently large $n$, $W_{n,T}^*\ge 1/2$. 
    By simple algebra calculations and $\lim_{n\to\infty} n_t/n = 1/(T-k+1)$, we have
    \bas
     & &\lim_{n\to\infty} M_{21}^*\le   \lim_{n\to\infty} \frac{1}{W_{n}^*}\sup_{x\in\mathbb{R}}\left|  \sum_{t=0}^{T-k}\frac{n_t}{n}\left\{B_{\rm emp}(x\mid t, \mathcal{D},\widetilde{\mathcal{D}}_{\rm cal})-B(x\mid t, \mathcal{D},\widetilde{\mathcal{D}}_{\rm cal})\right\}  \right| 
      \\
      &&\quad\quad\quad + \lim_{n\to\infty} \sup_{x \in \mathbb{R}} \left| \frac{1}{W_{n}^*}\sum_{t=0}^{T-k}\frac{n_t}{n}B(x\mid t,\mathcal{D},\widetilde{\mathcal{D}}_{\rm cal}) -B(x\mid \mathcal{D},\widetilde{\mathcal{D}}_{\rm cal})\right| =0\quad\text{by  \eqref{equ:tupe-splitting1}} ~and~ \eqref{equ:tupe-splitting2}.
      \\
      %&&\quad\quad\quad + \lim_{n\to\infty} \left(\frac{1}{W_{n}^*}-1\right)\sup_{x \in \mathbb{R}} \left|  B(x\mid \mathcal{D},\widetilde{\mathcal{D}}_{cal})\right| =0\quad\text{by Equations \eqref{equ:tupe-splitting1}} ~and~ \eqref{equ:tupe-splitting2}.
    \eas
 The desired result in Theorem 1 follows immediately.

   \end{proof}

\section{Proof of Theorem 2}

This section proves Theorem 2, which analyzes the coverage probability of the proposed PIs in the context of off-policy evaluation.  
{ We focus on the case where data splitting is performed in a trajectory-wise manner, and let \(n\) denote the number of trajectories in \(\mathcal{D}_{\text{cal}}\). Please note that, with a slight abuse of notation, $n$ here denotes the number of trajectories, which differs from its definition in the main paper. In the main paper, $n$ refers to the cardinality of the calibration set $\mathcal{D}_{\mathrm{cal}}$, where data are stored as tuples rather than trajectories. The result can be readily extended to the tuple-data-splitting setting, as discussed in the proof of Theorem 1.}

\begin{proof}[Proof of Theorem 2]
		Since $\widehat{C}_{N,\alpha}^{\rm off}(S_{\rm test})$ combines $B$ intervals following \cite{zhang2023conformal,solari2022multi}, it suffices to prove the validity of each CP interval. 
	With some abuse of notation, we denote the single CP interval at target coverage level $1-\alpha$ as $\widehat C_{N,\alpha}^{\rm off}(S_{\rm test})$.
	
	{First, we index the data points in the calibration dataset $\mathcal{D}_{\rm cal}$ as $\{1,2,\ldots,n\}$.} Given $\mathcal{D}$, $\widetilde{\mathcal{D}}_{\rm cal}$ is a set of samples drawn from the distribution
	\bas
	  \widehat F_{n} (s,g)= 
	  \sum_{t=0}^{T-k} \sum_{i=1}^{n} 
	  \frac{
	  	\widehat w_{\rm off}(\mathcal{H}_{i,t:t+k})
	  }
	  {
	  	 \sum_{t=0}^{T-k} \sum_{j=1}^{n}
	  	\widehat w_{\rm off}(\mathcal{H}_{j,t:t+k})
	  } 
	  I(S_{it}\le s, \widetilde G_{it}^{(k)}\le g),
	\eas
	 where $\mathcal{H}_{i,t:t+k}=(S_{it},A_{it},\cdots,S_{i,t+k})$ denotes the local trajectory segment following the behavior policy. Following the main idea of the proof of Theorem 1, we consider two new test points:
	\bas
	 (\widehat S_{\rm test}, \widehat{G}_{\rm test})
	  \sim \widehat F_{n} (s,g)
	\eas
	and
	\bas
	(S_{\rm test}, \widetilde G_{\rm test}) 
	\sim \mathcal{P}_{S_0} \times \left( (\mathcal{T}^{\pi})^k \widehat\eta^{\pi} \right)(S_0)
	\eas
	which are drawn independently.
	Then for $G_{\rm test} := G^{\pi}(S_{\rm test})$, we have
	\bas
	& & {\rm Pr} \left( G_{\rm test} \in \widehat{C}_{N,\alpha}^{\rm off}(S_{\rm test}) \right)\ge  {\rm Pr} \left( \widehat G_{\rm test} \in \widehat{C}_{N,\alpha}^{\rm off}(\widehat S_{\rm test}) \right)\\
	& &\quad 
	- \left|
	{\rm Pr} \left( \widehat G_{\rm test} \in \widehat{C}_{N,\alpha}^{\rm off}(\widehat S_{\rm test}) \right)
	- {\rm Pr} \left( \widetilde G_{\rm test} \in \widehat{C}_{N,\alpha}^{\rm off}(S_{\rm test}) \right)
	\right|  \\
	& &\quad - \left|
	{\rm Pr} \left( \widetilde G_{\rm test} \in \widehat{C}_{N,\alpha}^{\rm off}(S_{\rm test}) \right)
	- {\rm Pr} \left(  G_{\rm test} \in \widehat{C}_{N,\alpha}^{\rm off}(S_{\rm test}) \right)
	\right|  \\
	& &:=  \widetilde{M}_1 - \widetilde{M}_2 -\widetilde{M}_3.
	\eas

    Note that the dataset $\mathcal{D}$ is sampled from the behavior policy $\pi_b$ while $(S_{\rm test}, G_{\rm test})$ is generated by the target policy $\pi$.
     We now analyze $\widetilde M_1$, $\widetilde M_2$ and $\widetilde M_3$ separately.
     
	 \noindent\textbf{(1)} Given $\mathcal{D}$,  $(\widehat S_{\rm test}, \widehat{G}_{\rm test})$ is exchangeable with $\widetilde{\mathcal{D}}_{\rm cal}$.
    Existing result on coverage rate of SCP interval \citep{lei2018distribution} gives
	\ba
	\widetilde M_1 
	=  \e\left[
	{\rm Pr} \left( \widehat G_{\rm test} \in \widehat{C}_{N,\alpha}^{\rm off}(\widehat S_{\rm test}) 
	\,|\, \mathcal{D}\right)  \right] \ge 1 - \alpha.
	\label{Thm2_M1}
	\ea
    
	 \noindent\textbf{(2)}  Similar to the treatment of $M_3$ in the proof of Theorem 1,  we have
     \ba
      \widetilde{M}_3 
    \le\e \left[ \sqrt{2L\bar W_1((\mathcal{T}^{\pi})^k\widehat\eta^{\pi}, (\mathcal{T}^{\pi})^k \eta^{\pi} ) }\right]\le \sqrt{2L \gamma^k \e [\bar W_1(\widehat\eta^{\pi},\eta^{\pi})]}.\label{Thm2_M3}
     \ea

	\noindent\textbf{(3)} 
    Let   { $\mathcal{P}_t(s_0,a_0,\ldots,s_{k},g)$} denote the joint probability distribution of  { $(\mathcal{H}_{t:t+k}, \widetilde G_t^{(k)})$} given $\mathcal{D}_{tr}$ with some abuse of notation. Note that here { $(\mathcal{H}_{t:t+k}, \widetilde G_t^{(k)})$} is generated by $\pi_b$, consistent with the data.  We further denote $h_{0:k}:=(s_0,a_0,\ldots,s_{k})$ for notational simplicity. Then
    
%    Denote the joint probability distribution of  $(\mathcal{H}_{t:t+k+1}, \widetilde G_t^{(k)})$ given $\mathcal{D}_{tr}$ as $\mathcal{P}_t(s_0,a_0,\ldots,s_{k+1},g)$, with some abuse of notation. Note that here $\mathcal{H}_{t:t+k+1}$ is generated by $\pi_b$, consistent with the data.  We further denote $h_{0:k}:=(s_0,a_0,\ldots,s_{k})$ for notational simplicity. Then
	\bas
	\widetilde{M}_2(\mathcal{D}, \widetilde{\mathcal{D}}_{\rm cal})
    &:=&  \left|
	 {\rm Pr} \left( \widehat G_{\rm test} \in \widehat{C}_{N,\alpha}^{\rm off}( \widehat S_{\rm test}) 
	\,|\, \mathcal{D}, \widehat{\mathcal{D}}_{\rm cal} \right)
	- {\rm Pr} \left( \widetilde G_{\rm test} \in \widehat{C}_{N,\alpha}^{\rm off}(S_{\rm test}) 
	\,|\, \mathcal{D}, \widetilde{\mathcal{D}}_{\rm cal} \right)
	\right| 
	\\
	&=& \left|
	 {\rm Pr}\left( 
	 |\widehat G_{\rm test} - \widehat v^{\pi}( \widehat S_{\rm test})| \le \widehat q_{1-\alpha} 
	 \,|\, \mathcal{D}, \widetilde{\mathcal{D}}_{\rm cal} \right) \right. \\
	 &&\quad\quad\quad\quad\quad - \left.{\rm Pr}\left( 
	 |\widehat G_{\rm test} - \widehat v^{\pi}(\widehat S_{\rm test})| \le \widehat q_{1-a}  
	 \,|\, \mathcal{D},\widetilde{\mathcal{D}}_{\rm cal} \right)
	\right| 
%    && = \Bigg|
%\sum_{t=0}^{T-k} \sum_{i=1}^{n} 
%\frac{\widehat w_{\rm off}(\mathcal{H}_{i,t:t+k})}
%     {\sum_{t=0}^{T-k} \sum_{j=1}^{n} \widehat w_{\rm off}(\mathcal{H}_{j,t:t+k})}
%I\left\{ |\widetilde G_{it}^{(k)} - \widehat v^\pi(S_{it}) | \le \hat q_{1-%\alpha} \right\} \\
%&&\quad\quad\quad\quad\quad\quad- {\rm Pr} \left( \widetilde G_{\rm test} \in %\widehat{C}_{N,\alpha}^{\rm off}(S_{\rm test}) 
%	\,|\, \mathcal{D}, \widetilde{\mathcal{D}}_{cal} \right)
%\Bigg|
    \le\widetilde M_{21} + \widetilde M_{22}, 
    \eas
  where
\begin{align*}
&\widetilde M_{21} 
:= \sup_{x \in \mathbb{R}} \Bigg|
\sum_{t=0}^{T-k} \sum_{i=1}^{n} 
\frac{\widehat w_{\rm off}(\mathcal{H}_{i,t:t+k})}
     {\sum_{t=0}^{T-k} \sum_{j=1}^{n} \widehat w_{\rm off}(\mathcal{H}_{j,t:t+k})}
I\left\{ |\widetilde G_{it}^{(k)} - \widehat v^\pi(S_{it}) | \le x \right\} 
- B^{\rm off}(x \mid \mathcal{D}, \widetilde{\mathcal{D}}_{\rm cal})
\Bigg|, \\
&\widetilde M_{22} 
:= \left|
B^{\rm off}(\hat q_{1-\alpha} \mid \mathcal{D}, \widetilde{\mathcal{D}}_{\rm cal}) 
- \Pr\left( 
|\widetilde G_{\rm test} - \widehat v^\pi(S_{\rm test})| 
\le \widehat q_{1-\alpha} \,\middle|\, 
\mathcal{D}, \widetilde{\mathcal{D}}_{\rm cal} 
\right)
\right|,  \\
&B^{\rm off}(x \mid \mathcal{D}, \widetilde{\mathcal{D}}_{\rm cal}) 
:= \frac{1}{T - k + 1} \sum_{t=0}^{T - k} 
\int \widehat w_{\rm off}(h_{0:k}) \cdot 
I\left\{ |g - \widehat v^\pi(s_0)| \le x \right\} \, 
d\mathcal{P}_t(h_{0:k}, g).
\end{align*}

    \noindent\textbf{(3.1)} To analyze $\widetilde{M}_{21}$, we first define the normalization constant for weights as
    $$W_{n}^{\rm off}=\frac{1}{n(T-k+1)} \sum_{t=0}^{T-k} \sum_{i=1}^{n} 
	\widehat w_{\rm off}(\mathcal{H}_{i,t:t+k}).$$
    Thus the first term in $\widetilde{M}_{21}$ becomes
    \bas
\frac{1}{W_{n}^{\rm off}}\frac{1}{n(T-k+1)}\sum_{t=0}^{T-k} \sum_{i=1}^{n} 
\widehat w_{\rm off}(\mathcal{H}_{i,t:t+k})I\left\{ |\widetilde G_{it}^{(k)} - \widehat v^\pi(S_{it}) | \le x \right\}:=\frac{1}{W_{n}^{\rm off}}B_{\rm emp}^{\rm off}(x\mid \mathcal{D},\widetilde{\mathcal{D}}_{\rm cal}),
    \eas
 where  $B_{\rm emp}^{\rm off}(x\mid \mathcal{D},\widetilde{\mathcal{D}}_{\rm cal})$ is the empirical version of $B^{\rm off}(x\mid \mathcal{D},\widetilde{\mathcal{D}}_{\rm cal})$. By a simple algebraic calculation, we have
 \begin{align*}
\widetilde{M}_{21} 
&\le \frac{1}{W_{n}^{\rm off}} \sup_{x \in \mathbb{R}} 
\left| B^{\rm off}_{\rm emp}(x \mid \mathcal{D}, \widetilde{\mathcal{D}}_{\rm cal}) 
- B^{\rm off}(x \mid \mathcal{D}, \widetilde{\mathcal{D}}_{\rm cal}) \right| \\
&\quad + \left( \frac{1}{W_{n}^{\rm off}} - 1 \right) 
\sup_{x \in \mathbb{R}} \left| B^{\rm off}(x \mid \mathcal{D}, \widetilde{\mathcal{D}}_{\rm cal}) \right|.
\end{align*}
    As $\frac{1}{T-k+1}\sum_{t=0}^{T-k}\e[\widehat w_{\rm off}(\mathcal{H}_{t:t+k})\mid\mathcal{D}_{\rm tr}]=1$, by law of large numbers, $ \lim_{n \to \infty}  W_{n}^{\rm off}=1$. Hence, for sufficiently large $n$, $W_{n}^{\rm off}\ge 1/2$ and
\begin{align*}
\widetilde{M}_{21} \le 2 \underbrace{\sup_{x \in \mathbb{R}} 
\left| B^{\rm off}_{\rm emp}(x \mid \mathcal{D}, \widetilde{\mathcal{D}}_{\rm cal}) 
- B^{\rm off}(x \mid \mathcal{D}, \widetilde{\mathcal{D}}_{\rm cal}) \right|}_{\widetilde{E}}  +  \underbrace{ { \left| \frac{1}{W_{n}^{\rm off}} - 1 \right| } \sup_{x \in \mathbb{R}}\left| B^{\rm off}(x \mid \mathcal{D}, \widetilde{\mathcal{D}}_{\rm cal}) \right|}_{\widetilde{F}}.
\end{align*}
\noindent\textbf{(3.1.1)} For $\widetilde{E}$, since $\e[\widehat w_{\rm off}(\mathcal{H}_{t:t+k})] < \infty$ for $0 \le t \le T-k$, the function class $\{\widehat w_{\rm off}(h_{0:k},g)I\{|g-\widehat{v}^\pi(s_0)|\le x\}: x \in \mathcal{R}\}$ is { $\{\mathcal{P}_t(h_{0:k},g): 0\le t \le T-k\}$}-Glivenko-Cantelli. Applying the same argument as in Equation~\eqref{equ:average-over-time-1}, we obtain

    \begin{equation} 
    \lim_{n \to \infty} \sup_{x \in \mathbb{R}} 
\Big| B_{\rm emp}^{\rm off}(x\mid \mathcal{D},\widetilde{\mathcal{D}}_{\rm cal})-B^{\rm off}(x\mid \mathcal{D},\widetilde{\mathcal{D}}_{\rm cal}) \Big| = 0.\label{equ:emp-off}
   \end{equation}
\noindent\textbf{(3.1.2)} For $\widetilde{F}$, we have { $\lim_{n\to\infty}\left( {1}/{W_{n}^{\rm off}} - 1 \right) = 0$, and}
\bas
\sup_{x \in \mathbb{R}} \left| B^{\rm off}(x \mid \mathcal{D}, \widetilde{\mathcal{D}}_{\rm cal}) \right|&\le& \frac{1}{T - k + 1} \sum_{t=0}^{T - k} 
\int \widehat w_{\rm off}(h_{0:k})  
d\mathcal{P}_t(h_{0:k+1}, g)\nonumber\\
&=&\frac{1}{T - k + 1} \sum_{t=0}^{T-k} 
\e\left[ \widehat{w}_{\rm off}(\mathcal{H}_{t:t+k}) \mid \mathcal{D}_{\rm tr} \right]=1. 
\eas
Combining these results with \eqref{equ:emp-off}, we obtain
\begin{equation}
\lim_{n \to \infty} \widetilde{M}_{21} = 0.
\label{Thm1_M21}
\end{equation}

\noindent\textbf{(3.2) Bound on $\widetilde M_{22}$.} Following the proof of Theorem 1, we  define a new probability measure
$$
\frac{1}{T-k+1}\sum_{t=0}^{T-k}\widehat{w}_{\rm off}(h_{0:k})d\mathcal{P}_{t}(h_{0:k},g),
$$
and let $(\widetilde{\mathcal{H}}_{0:k},\widetilde{G})$ be drawn from this measure with $\widetilde{\mathcal{H}}_{0:k}=(\widetilde{S}_{0},\widetilde{A}_{0},\cdots,\widetilde{S}_{k})$. Then $\widetilde M_{22}$ can be equivalently written as
\bas
\widetilde M_{22} 
:= \left|
\Pr\left( 
|\widetilde G - \widehat v^\pi(\widetilde{S}_0)| 
\le \widehat q_{1-\alpha} \,\middle|\, 
\mathcal{D}, \widetilde{\mathcal{D}}_{\rm cal} 
\right) 
- \Pr\left( 
|\widetilde G_{\rm test} - \widehat v^\pi(S_{\rm test})| 
\le \widehat q_{1-\alpha} \,\middle|\, 
\mathcal{D}, \widetilde{\mathcal{D}}_{\rm cal} 
\right)
\right|.
\eas
Denote the marginal distribution of $\mathcal{P}_{t}(h_{0:k},g)$ as $\mathcal{P}_{t}(h_{0:k})$ and define the calibration marginal distribution as $\mathcal{P}_{\rm cal}(h_{0:k})=\frac{1}{T-k+1}\sum_{t=0}^{T-k}\mathcal{P}_{t}(h_{0:k})$. Then $\widetilde{\mathcal{H}}_{0:k}\sim \widehat{w}_{\rm off}(h_{0:k})\mathcal{P}_{\rm cal}(h_{0:k})$, and the unobserved $\mathcal{H}_{\rm test, 0:k}=(S_{\rm test,0},A_{\rm test,0},\cdots,S_{\rm test,k})\sim w_{\rm off}(h_{0:k})\mathcal{P}_{\rm cal}(h_{0:k})$,  where $S_{\rm test,0}=S_{\rm test}$.

Since the conditional distributions $\widetilde{G}\mid \widetilde{\mathcal{H}}_{0:k}$ and $\widetilde{G}_{\rm test}\mid {\mathcal{H}}_{\rm test, 0:k}$ are the identical, by Eq. (A.9) in \cite{lei2021conformal}, we have
\bas
  \widetilde{M}_{22} 
     &\le& d_{TV}(\mathcal{P}_{\widetilde{\mathcal{H}}_{0:k}},\mathcal{P}_{\mathcal{H}_{\rm test,0:k}})
     \notag \\
&\le& \frac{1}{2(T-k+1)}\sum_{t=0}^k
     \e\left[ 
     \left|\widehat w_{\rm off}(\mathcal{H}_{t:t+k}) - w_{\rm off}(\mathcal{H}_{t:t+k})\right|  \mid\mathcal{D}_{\rm tr}
     \right].
     \label{Thm2_M22}
\eas
The desired result follows by combining (\ref{Thm2_M1}) - (\ref{Thm2_M22}).

\end{proof}

\section{Algorithm for Off-Policy Setting}

Algorithm  \ref{CP_off_policy}  presents the proposed algorithm for the off-policy setting, 
which closely parallels that in the on-policy case.

\begin{center}
	\begin{algorithm}[ht] \label{CP_off_policy} 
			\caption{ {\it  CP for Infinite Horizon Off-policy Evaluation   } }
			\KwData{  
				$\mathcal{D} = \{ (S_{it},A_{it},R_{it},S_{i,t+1}) : 1 \le i \le N, 1 \le t \le T \}$, a test initial state $S_{\rm test}$ and a target policy $\pi$.
	           }
			
			\KwIn{ $1-\alpha$, target coverage level; 
				$\widetilde{\mathcal{A}}$, an off-policy distributional RL algorithm; 
				$\mathcal{B}$, a propensity score training algorithm;
				${\mathcal{W}}$, a density ratio estimation algorithm; 
				$k$, step width;
				$B$, resampling number;
				$l$, subsample size;
				$\xi$, multiple subsampling parameter
			}

			\KwOut{Prediction interval for 
            $G^\pi(S_{\rm test})$}

			%\bigskip
			%\textbf{Prepare:}
			Split the data: 
                 $\mathcal{D} 
                  = \mathcal{D}_{\rm tr} \bigcup \mathcal{D}_{\rm cal}$
                 where 
                 $ \mathcal{D}_{\rm tr} = \left\{(S_{it},A_{it},R_{it},S_{i,t+1}) : (i,t) \in \mathcal{I}_{\rm tr} \right\}$
                  and
                 $\mathcal{D}_{\rm cal} = \left\{ (S_{it},A_{it},R_{it},\ldots,S_{i,t+k}) : (i,t) \in \mathcal{I}_{\rm cal} \right\}$. Here, $\mathcal{I}_{{\rm tr}}$ and $\mathcal{I}_{{\rm cal}}$ denote the indices of transitions  in the training and calibration datasets, respectively.

			Train a conditional return model 
			$\widehat \eta^{\pi}(s)$ 
			using $\widetilde{\mathcal{A}}$ based on $\mathcal{D}_{\rm tr}$.

                Obtain the value function estimator $\widehat v^{\pi}(s)$, the expectation of $\widehat\eta^{\pi}(s)$.

			Obtain $\widehat w_{\rm on}(s)$ as an estimator of the density ratio (2) in the main paper based on $\left\{ S_{i0}: (i,0) \in \mathcal{I}_{\rm tr} \right\}$ 
            and $\left\{ S_{it}: (i,t) \in \mathcal{I}_{\rm tr} \right\}$ using $\mathcal{W}$.
			
			Train $\widehat \pi^{b}(a \,|\, s)$ based on  
			$\{(S_{it},A_{it}): (i,t) \in \mathcal{I}_{\rm tr}\}$ using $\mathcal{B}$.
			
			Obtain $\widehat w_{\rm off}(\cdot)$ by plugging in $\widehat w_{\rm on}$ and $\widehat\pi_b$ in (3) of the main paper.

			\For{$b=1:B$} {
				\bit
				\item
		      Sample $l$ data tuples
                $\{(S_{it}, {A}_{i,t}, {R}_{i,t}, \ldots,S_{i,t+k}) : (i,t) \in \mathcal{I}_{\rm cal}^{(b)}\}$ from $\mathcal{D}_{\rm cal}$ accoring to the importance weight $\widehat w_{\rm off}(S_{it},A_{it},\ldots,S_{i,t+k})$.

                \item
                Calculate pseudo-return  (1) in the main paper and obtain $\widetilde{\mathcal{D}}_{\rm cal}^{(b)}
                 := \{(S_{it}, \widetilde{G}_{it}^{(k)}) : (i,t) \in \mathcal{I}_{\rm cal}^{(b)}\}$.

                 \item
		       Calculate the nonconformity scores: 
		       $\{
		          V_{it}
		          :=  | \widetilde{G}_{it}^{(k)} - \widehat v^{\pi}(S_{it}) |  : (i,t) \in \mathcal{I}_{\rm cal}^{(b)}\}
		          \}$.

				\item 
				Calculate $\widehat{q}_{1-\alpha\xi}^{(b)} $, the $\lceil l(1-\alpha\xi)\rceil$-th smallest value of $\{V_{it}: (i,t) \in \mathcal{I}_{\rm cal}^{(b)}\}$.

				\item Obtain 
				$\widehat C_{N,\alpha\xi}^{(b)}(S_{\rm test}) 
				= \widehat v^{\pi}(S_{\rm test}) \pm \hat q_{1-\alpha\xi}^{(b)} $.

				\eit
			}

			\KwResult{
				A conformal predictive region for $G^\pi(S_{\rm test})$ with a coverage rate of $1-\alpha$  is
				\ba
				\widehat{C}_{N,\alpha}^{\rm off} (S_{\rm test})
				&=&
				\left\{ G: \frac1B \sum_{b=1}^B 
				I\left\{ G \in \widehat{C}_{N,\alpha\xi}^{(b)}(S_{\rm test})  \right\} \ge 1-\xi 
				\right\}.
				\label{PI_off_policy}
				\ea
			}
	\end{algorithm}
\end{center}

\section{Implementation Details and Additional Results}

We provide additional implementation details for the numerical experiments. The code is available at:  \href{https://github.com/yyzhangecnu/CPbeyonghorizon}{https://github.com/yyzhangecnu/CPbeyonghorizon}.

\paragraph{Example 1.} We adopt the QTD algorithm (Algorithm 1 in \cite{rowland2024analysis}) to estimate the quantiles of the return distribution. The learning rate $\rho$ is set to 0.1, and the discount factor $\gamma$ is 0.8. We use 20 quantile levels in the estimation. The behavior policy is estimated based on the empirical frequency of \((s, a)\) pairs in the training set, and the importance weights are computed similarly using frequency-based estimates. The hyperparameter \(\xi\), which controls the aggregation of multiple prediction intervals, is selected via grid-based cross-fitting since simulations allow us to generate trajectories with sufficiently large $T$ to get accurate return. We set the number of aggregated intervals to \(B = 100\), with each interval constructed from a subsample of 400 tuples drawn from the calibration dataset. We repeat the experiment over 100 simulation runs and report the boxplots of the empirical coverage probabilities and the average lengths of PIs. 
The nominal coverage level is fixed at $90\%$.

\noindent\textbf{Influence of $k$.} 
Based on Example 1, we further investigate the effect of using larger $k$ values, specifically for $k = 6, 7, 8$. Each experiment is repeated 100 times, and we report the mean and standard deviation of the empirical coverage probability ({cov}) and prediction interval length ({len}) under the nominal $90\%$ coverage level. 

As shown in Table~\ref{tab:choice-of-k}, increasing $k$ consistently results in overcoverage and, consequently, wider prediction intervals. This observation aligns with our theoretical results in Section~4 (Theorems~1 and~2), which reveal an inherent trade-off. A larger $k$ reduces the approximation error in estimating $\widehat{\eta}^{\pi}$, but at the same time, it increases the difficulty of accurately estimating the off-policy weights and maintaining the approximate independence of calibration samples particularly under substantial distributional shifts. Empirically, we find that choosing $k=2$ or $3$ provides a good balance between these competing factors.

\begin{table}[htbp]
\centering
\caption{Coverage (cov) and average length (len) for different $k$ under on-policy and off-policy settings with $\xi=0.8$. Standard errors are shown in parentheses.}
\label{tab:onoffpolicy}
\begin{adjustbox}{width=\textwidth, center}
\begin{tabular}{lcccccccc}
\toprule
\textbf{on} & $k=1$ & $k=2$ & $k=3$ & $k=4$ & $k=5$ & $k=6$ & $k=7$ & $k=8$ \\
\midrule
cov  & 0.87(0.01) & 0.90(0.01) & 0.91(0.01) & 0.92(0.01) & 0.92(0.01) & 0.93(0.01) & 0.94(0.01) & 0.94(0.01) \\
len  & 7.78(0.10) & 8.24(0.10) & 8.56(0.13) & 8.78(0.14) & 9.00(0.15) & 9.15(0.19) & 9.31(0.23) & 9.50(0.22) \\
\midrule
\textbf{off} & $k=1$ & $k=2$ & $k=3$ & $k=4$ & $k=5$ & $k=6$ & $k=7$ & $k=8$ \\
\midrule
cov  & 0.87(0.01) & 0.91(0.01) & 0.92(0.01) & 0.92(0.01) & 0.93(0.01) & 0.93(0.01) & 0.93(0.01) & 0.94(0.02) \\
len & 7.57(0.10) & 8.13(0.11) & 8.47(0.14) & 8.67(0.14) & 8.90(0.17) & 9.02(0.18) & 9.20(0.18) & 9.26(0.20) \\
\bottomrule
\end{tabular}\label{tab:choice-of-k}
\end{adjustbox}
\end{table}

\noindent\textbf{Influence of $\xi$.} We conduct experiments for Example 1 with $\xi$ varying from $0.1$ to $0.9$ and $k = 2, 3, 4$. Each setting is repeated 100 times, and we report the mean and standard deviation of the coverage probability (cov) and interval length (len) at the nominal $90\%$ coverage level, as shown in Table \ref{tab:example1_results}. The results show that smaller $\xi$ and larger $k$ tend to cause overcoverage, whereas settings with $\xi \ge 0.5$ and $k = 2, 3$ generally achieve satisfactory performance.

\begin{table}[htbp]
\centering
\small 
\caption{Coverage probability (cov) and interval length (len) for different $\xi$ under on-policy and off-policy settings. Standard errors are shown in parentheses.}
\label{tab:example1_results}
\begin{tabular}{l ccc ccc}
\toprule
\textbf{on} & \multicolumn{3}{c}{{cov}} & \multicolumn{3}{c}{{len}} \\
$\xi$ & $k=2$ & $k=3$ & $k=4$ & $k=2$ & $k=3$ & $k=4$ \\
\midrule
0.1 & 0.95(0.01) & 0.96(0.01) & 0.96(0.01) & 10.21(0.20) & 10.71(0.21) & 11.04(0.26) \\
0.2 & 0.95(0.01) & 0.95(0.01) & 0.95(0.01) & 9.68(0.15) & 10.10(0.16) & 10.40(0.17) \\
0.3 & 0.94(0.01) & 0.95(0.01) & 0.95(0.01) & 9.30(0.12) & 9.67(0.14) & 9.95(0.16) \\
0.4 & 0.92(0.01) & 0.94(0.01) & 0.95(0.01) & 8.98(0.10) & 9.34(0.13) & 9.62(0.15) \\
0.5 & 0.92(0.01) & 0.93(0.01) & 0.94(0.01) & 8.73(0.08) & 9.07(0.13) & 9.33(0.15) \\
0.6 & 0.91(0.01) & 0.92(0.01) & 0.93(0.01) & 8.53(0.09) & 8.87(0.13) & 9.09(0.16) \\
0.7 & 0.91(0.01) & 0.92(0.01) & 0.92(0.01) & 8.37(0.09) & 8.69(0.12) & 8.92(0.14) \\
0.8 & 0.90(0.01) & 0.91(0.01) & 0.92(0.01) & 8.24(0.10) & 8.56(0.13) & 8.78(0.14) \\
0.9 & 0.90(0.01) & 0.91(0.01) & 0.92(0.01) & 8.20(0.12) & 8.51(0.14) & 8.72(0.16) \\
\midrule
\textbf{off} & \multicolumn{3}{c}{{cov}} & \multicolumn{3}{c}{{len}} \\
$\xi$ & $k=2$ & $k=3$ & $k=4$ & $k=2$ & $k=3$ & $k=4$ \\
\midrule
0.1 & 0.96(0.01) & 0.96(0.01) & 0.97(0.01) & 10.17(0.19) & 10.68(0.22) & 10.95(0.30) \\
0.2 & 0.95(0.01) & 0.95(0.01) & 0.96(0.01) & 9.62(0.15) & 10.08(0.17) & 10.33(0.20) \\
0.3 & 0.94(0.01) & 0.95(0.01) & 0.96(0.01) & 9.24(0.12) & 9.64(0.15) & 9.90(0.17) \\
0.4 & 0.93(0.01) & 0.94(0.01) & 0.95(0.02) & 8.93(0.10) & 9.30(0.13) & 9.57(0.15) \\
0.5 & 0.92(0.01) & 0.93(0.01) & 0.94(0.01) & 8.68(0.11) & 9.03(0.14) & 9.28(0.15) \\
0.6 & 0.92(0.01) & 0.93(0.01) & 0.93(0.01) & 8.43(0.11) & 8.78(0.14) & 8.99(0.14) \\
0.7 & 0.91(0.01) & 0.92(0.01) & 0.93(0.01) & 8.26(0.11) & 8.61(0.13) & 8.82(0.14) \\
0.8 & 0.91(0.01) & 0.92(0.01) & 0.92(0.01) & 8.13(0.11) & 8.47(0.14) & 8.67(0.14) \\
0.9 & 0.91(0.01) & 0.92(0.01) & 0.92(0.01) & 8.07(0.13) & 8.41(0.16) & 8.61(0.15) \\
\bottomrule
\end{tabular}
\end{table}

\noindent\textbf{Comparison with \cite{foffano2023conformal}.} We compare the performance of our method and that of \cite{foffano2023conformal} in the off-policy setting for Example 1 with a fixed horizon of 20. For Foffano’s method, we follow their gradient-based approach to train the likelihood ratio model $w(x, y)$ via linear regression and apply WCP to construct prediction intervals. For our method, we replace the nonconformity score with the double-quantile (DQ) score from \cite{foffano2023conformal}, setting $\xi$ to 0.5 and 0.6, and $k$ to 2 and 3. To better accommodate the DQ score, we employ the interval aggregation technique proposed by \cite{meinshausen2009p}. Each experiment is repeated 100 times, with the nominal coverage level fixed at $90\%$. The results, shown in Figure~\ref{fig:ex-com}, indicate that our method achieves superior performance in terms of both coverage probability and average interval length.

\begin{figure}[ht]
	\centering
	\subfigure[Coverage probability]{\includegraphics[width=.35\columnwidth]{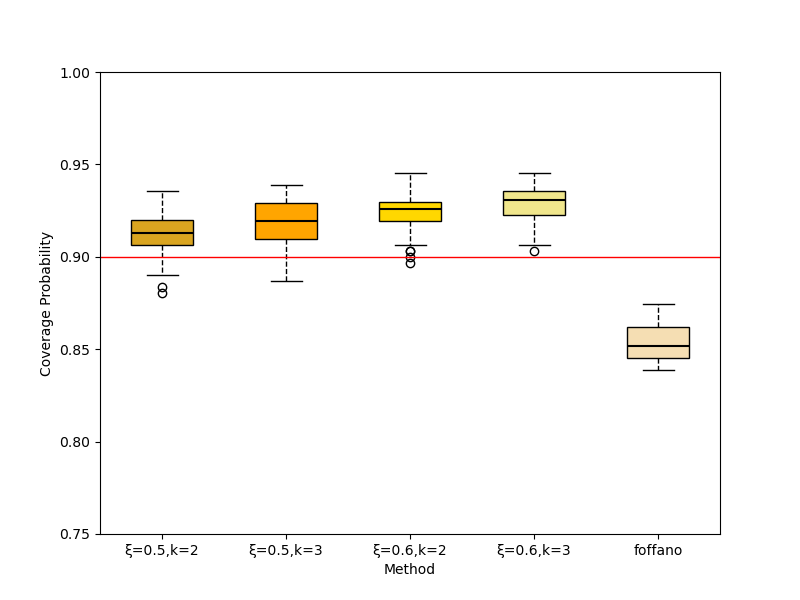}}
    \subfigure[Average length]{\includegraphics[width=.35\columnwidth]{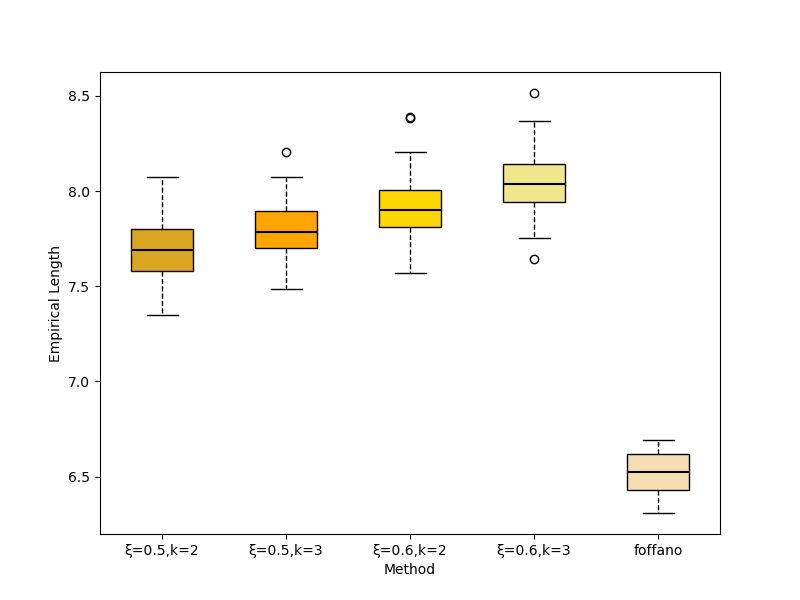}}
    \caption{{Coverage probability and average interval length at the 90\% level for the proposed method with $\xi=0.5,0.6$ and $k=2,3$ (from left to right) and Foffano's method (rightmost).}}\label{fig:ex-com}

\end{figure}

\paragraph{Example 2.} The state space is continuous in this setting. To apply the QTD algorithm, we train a quantile network with 20 quantile levels. The input to the network is the state, and the architecture consists of three layers with 32 hidden neurons and 40 output units, each corresponding to a specific quantile level for a given state-action pair. The behavior policy is estimated using a separate neural network with architecture $2 \rightarrow 32 \rightarrow 32 \rightarrow 2$, where the outputs represent the action probabilities. Following the QR-DQN algorithm in \cite{dabney2018distributional}, we replace the quantile regression loss with the Huber quantile loss to improve stability.

The importance weights are estimated using logistic regression. The hyperparameter \(\xi\), which governs the aggregation of multiple PIs, is selected via grid-based cross-fitting since simulations allow us to generate trajectories with sufficiently large $T$ to get accurate return. We set the number of aggregated intervals to \(B = 50\), with each interval constructed from a subsample of 200 tuples drawn from the calibration dataset. We repeat the experiment over 100 simulation runs and report boxplots of the empirical coverage probabilities and the average lengths of the resulting PIs. The nominal coverage level is fixed at 90\%.

%In experiment 2, the former approach fails since the state space is continuous. Instead, following the idea of QR-DQN(dabney2017), we train a quantile network for every state with output matrix in the form of action card*number(2*20) of quantiles to simulate $\theta(s, a, i)$ with structure $2\rightarrow 32\rightarrow 32\rightarrow 40$. We also estimate behavior policy via softmax(logistic function) layer with network structure $2\rightarrow 32\rightarrow 32\rightarrow 2(action\ card)$. In practice, we perform polyak averaging and warm-up+cosine annealing learning rate scheduler to boost training. As introduced in dabney2017, similar to quantile regression loss, we consider the pinball loss $L_{\theta}  =\mathbb{E}_X[|X-\theta||\tau-\delta_{X<\theta}|] =\mathbb{E}_X[\rho_\tau(X-\theta)]$ , where $\rho_{\tau}(u)=|u|\left|\tau-\delta_{u<0}\right|$ and we change $|u|$ into Huber loss for better optimization. $\theta$ can be any quantile of $G^\pi(s,a)$ and $X$ is $r(s,a)+\gamma \theta (s_next,a_next,j)$, a rough realization of $\hat{G}^\pi(s,a)$. We take average over $j=1,\ldots,m$ to approximate expectation. In the training process, after storing tuples in replay buffer, we sample a batch from replay buffer in each iteration step and calculate its total loss. In off-policy setting, we replace $a_{next}$ with pseudo-sample from target policy which is consistent to our approach in the experiment 1. The remaining procedure of the experiment is the same as before. 

\paragraph{Example 3.} Mountain car is a classic RL control problem.  We first use RBF-based feature engineering to search for a suboptimal policy denoted by $\pi_{Q}$ via Q-learning. To better illustrate that our proposal is a wrapper, we apply kernel density estimation to approximate the return distribution from Monte Carlo rollouts. The discount factor $\gamma$ is set to 0.99. The remaining procedure of the experiment is the same as Example 2. We set the number of aggregated intervals to \(B = 50\), with each interval constructed from a subsample of 200 tuples drawn from the calibration dataset. We repeat the experiment over 50 simulation runs and report boxplots of the empirical coverage probabilities and the average lengths of the resulting PIs. The nominal coverage level is fixed at 90\%.

Figure \ref{fig:ex3} presents the results for both on-policy and off-policy settings in Example 3. These experiments demonstrate that our proposed method consistently outperforms the kernel-density-based approach, even when the kernel density is estimated using Monte Carlo rollouts under the target policy. Notably, all intervals exhibit greater variance compared to those in Examples 1 and 2. This increased variance arises from the challenging nature of the environment, where the agent receives a constant reward of -1 until reaching the goal (the flag). As a result, the immediate reward provides limited information, making learning and accurate value estimation more difficult. 

\begin{figure}[ht]
	\centering
	\subfigure[Coverage probability]{
		\begin{minipage}[]{0.8\columnwidth}
			\centering
			\includegraphics[width=.3\columnwidth]{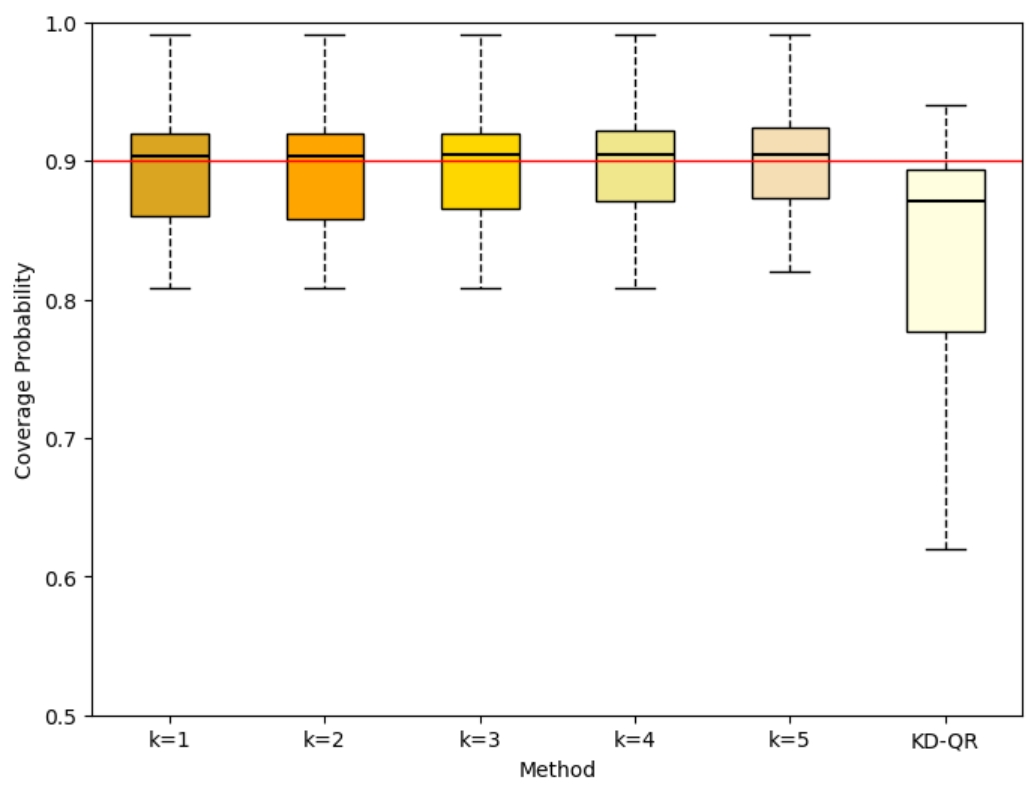}
			\includegraphics[width=.3\columnwidth]{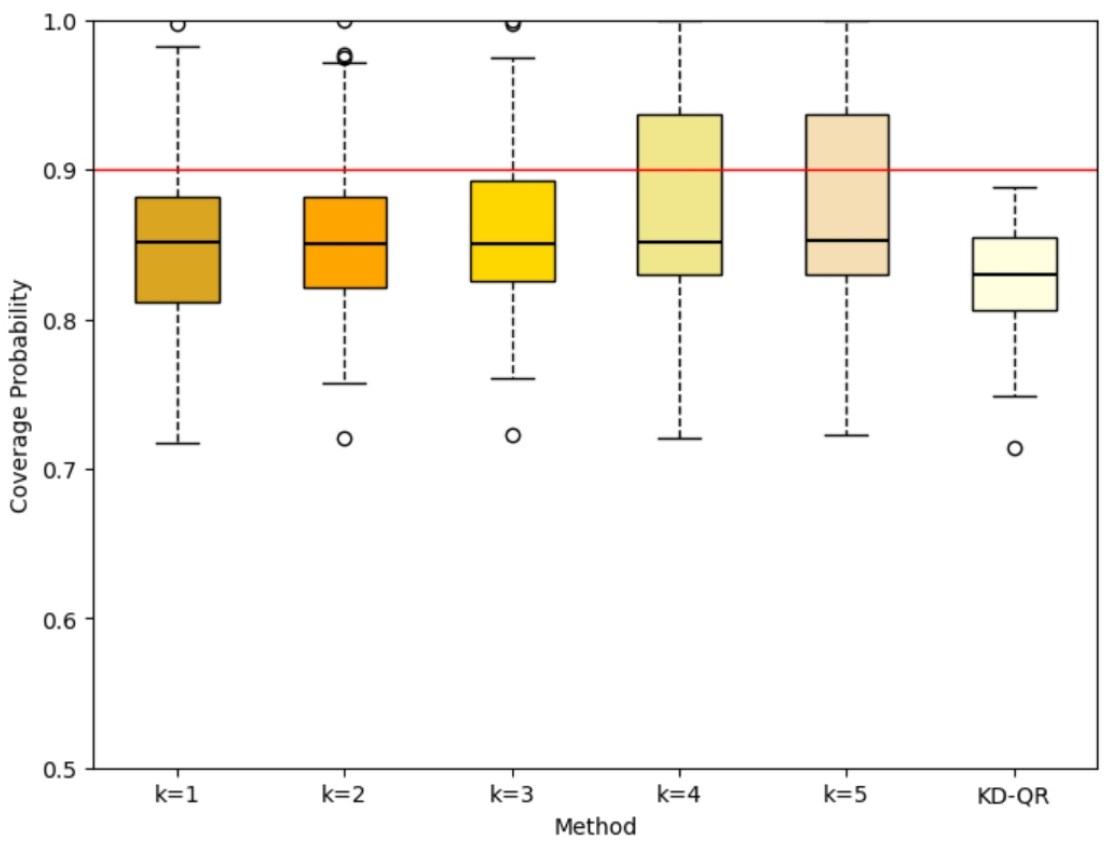}
		\end{minipage}
	}
	\subfigure[Average length]{
		\begin{minipage}[]{0.8\columnwidth}
			\centering
			\includegraphics[width=.3\columnwidth]{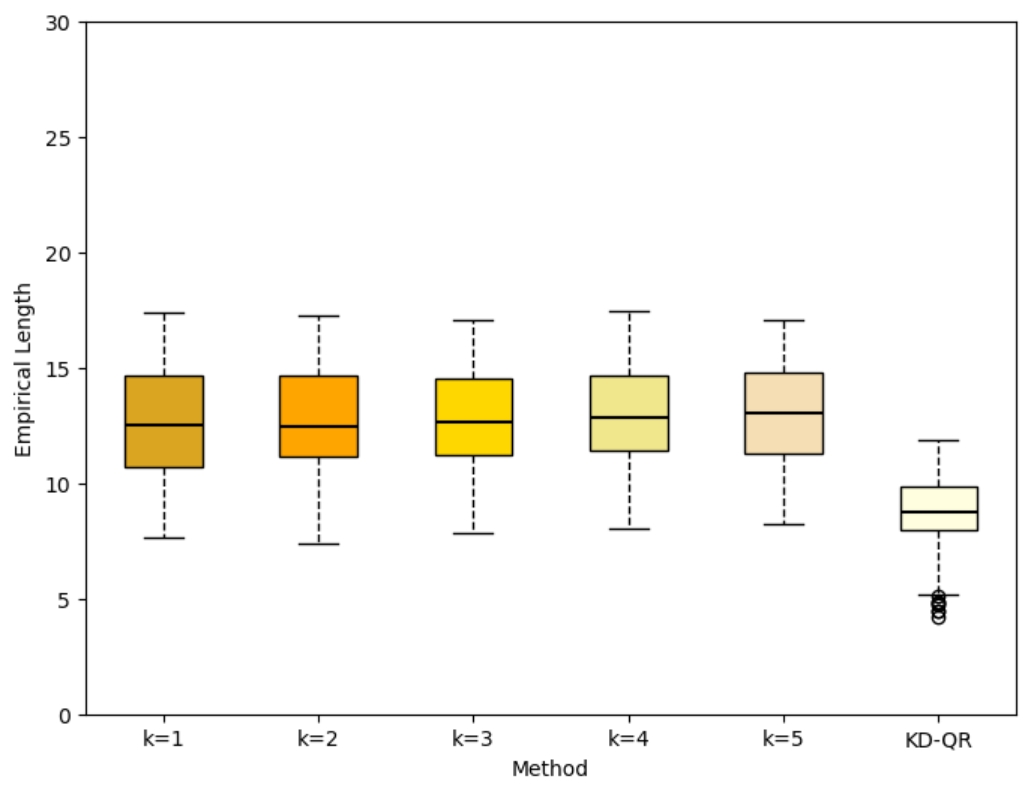}
			\includegraphics[width=.3\columnwidth]{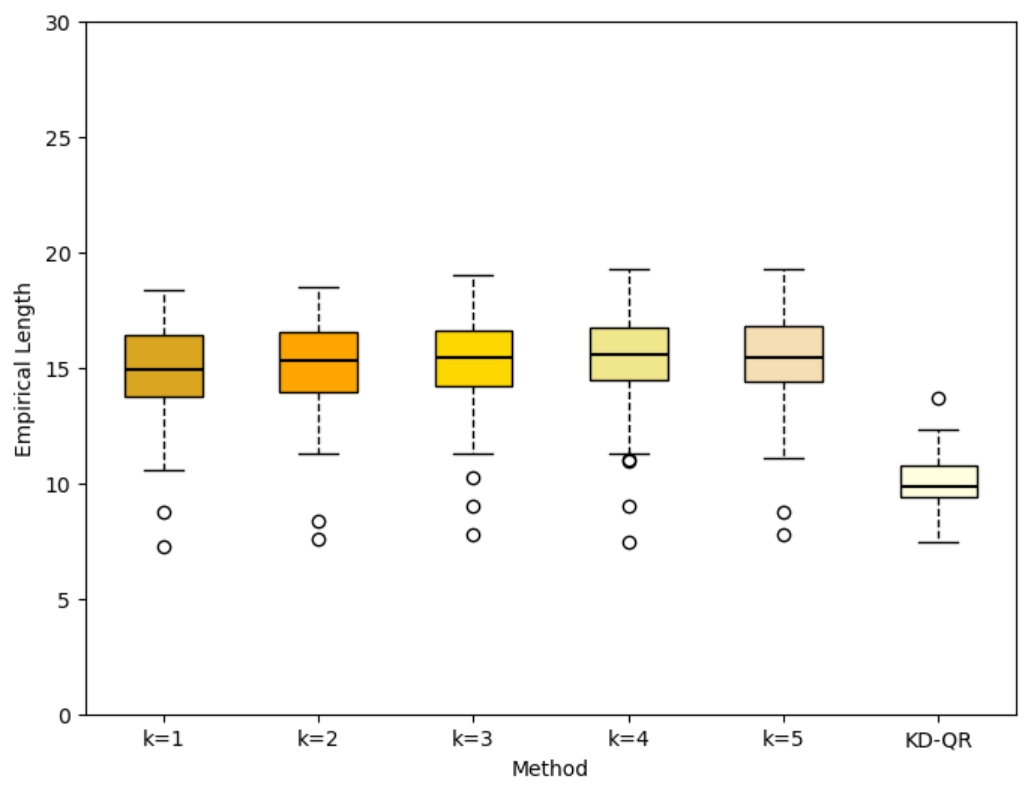}
		\end{minipage}
	}
    \caption{{Coverage probability and average interval length at the 90\% level for the proposed method with $k$-step pseudo-returns ($k = 1,\ldots,5$, from left to right) and KD-QR (rightmost), under on-policy (left) and off-policy (right) settings in Example 3.}}\label{fig:ex3}

\end{figure}

 \paragraph{Example 4.} We extend Example 1 to a high-dimensional setting with 50 states, denoted by $\mathbf{S}_{t}=(S_{1t},S_{2t},\cdots,S_{50t})^{\top}$, where each feature $S_{jt}$ for $1 \le j \le 50$ is binary, taking values $x_1$ or $x_2$. The action space is $\{0,1\}$ and only affects transitions of the first state $S_{1t}$. The remaining states independently take values $x_1$ or $x_2$ with equal probability at each time step, thereby serving as confounders. The agent, however, does not know which state is directly influenced by the action. The reward follows the same distribution as in Example 1. The behavior policy specifies transition probabilities of $0.4$ for $x_1 \rightarrow x_2$ and $0.8$ for $x_2 \rightarrow x_1$, while the target policy remains the same as in Example 1 for the off-policy setting.

 We employ quantile temporal difference (QTD) learning with linear regression and a ridge penalty to alleviate overfitting. The number of aggregated intervals is set to $B=50$ and the hyperparameter is fixed at $\xi=0.8$. Each interval is constructed from a subsample of 200 tuples drawn from 6000 calibration tuples. Experiments are conducted for $k = 1, \dots, 5$, each repeated 50 times. We report boxplots of the empirical coverage probabilities and average interval lengths in Figure~\ref{fig:ex4}, with the nominal coverage level fixed at $90\%$. The results show that our proposed method consistently outperforms the DRL-QR baseline in this high-dimensional setting.

 \begin{figure}[ht]
	\centering
	\subfigure[Coverage probability]{
		\begin{minipage}[]{0.8\columnwidth}
			\centering
			\includegraphics[width=.35\columnwidth]{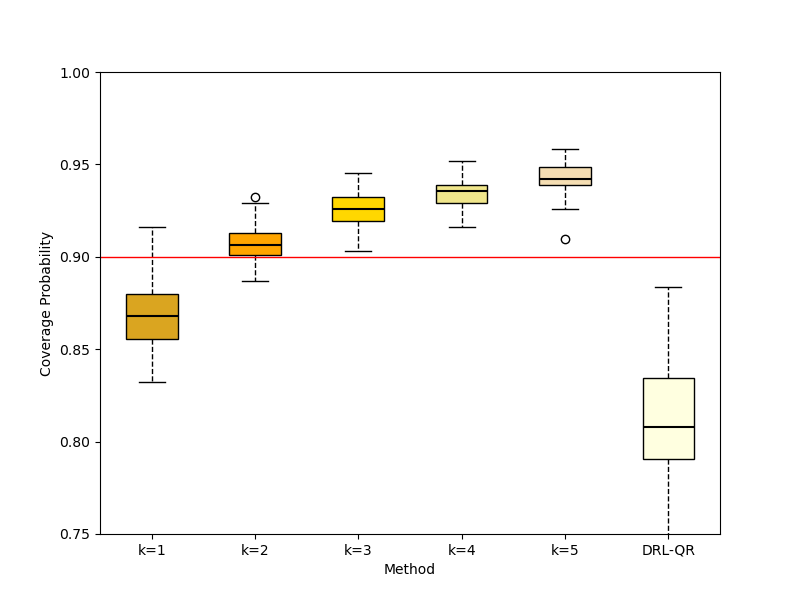}
			\includegraphics[width=.35\columnwidth]{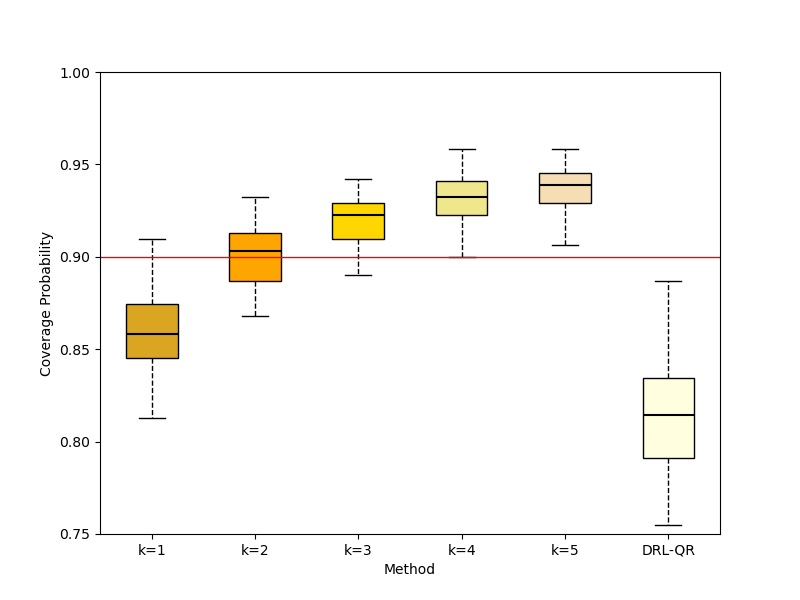}
		\end{minipage}
	}
	\subfigure[Average length]{
		\begin{minipage}[]{0.8\columnwidth}
			\centering
			\includegraphics[width=.35\columnwidth]{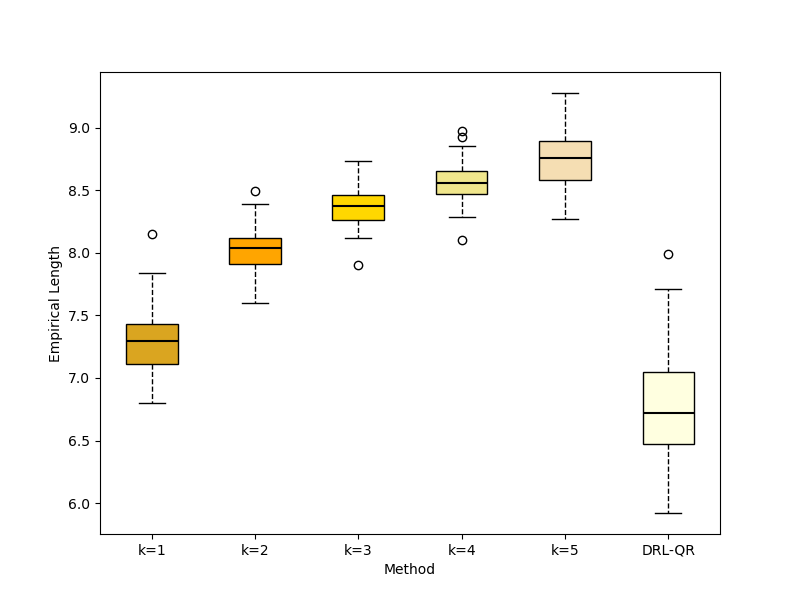}
			\includegraphics[width=.35\columnwidth]{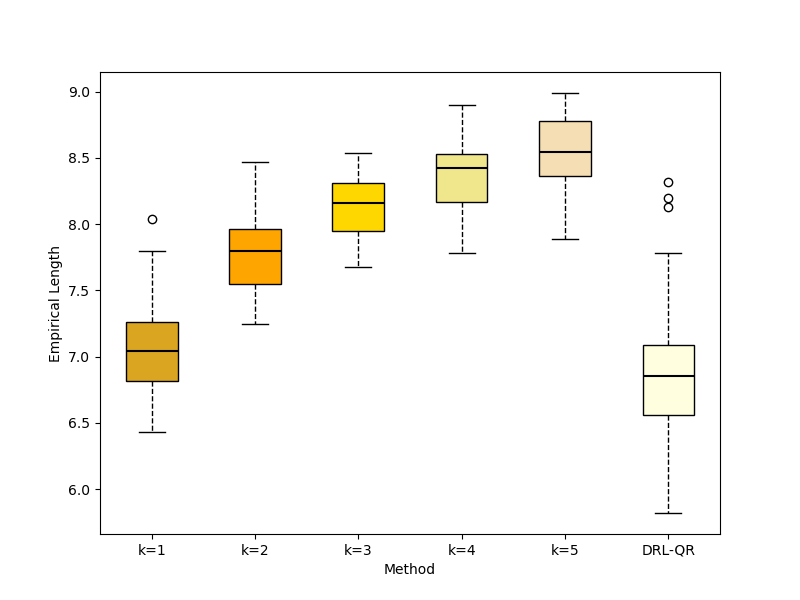}
		\end{minipage}
	}
    \caption{{Coverage probability and average interval length at the 90\% level for the proposed method with $k$-step pseudo-returns ($k = 1,\ldots,5$, from left to right) and DRL-QR (rightmost), under on-policy (left) and off-policy (right) settings in Example 4.}}\label{fig:ex4}

\end{figure}

\end{document}